\documentclass{article}

\usepackage{microtype}
\usepackage{graphicx}
\usepackage{subfigure}
\usepackage{booktabs} % for professional tables

\usepackage{hyperref}

\usepackage[accepted]{icml2019}

\usepackage[textsize=small]{todonotes}
\usepackage{enumitem}

\usepackage{xcite}

\usepackage{xr}
\makeatletter
\newcommand*{\addFileDependency}[1]{% argument=file name and extension
  \typeout{(#1)}% latexmk will find this if $recorder=0 (however, in that case, it will ignore #1 if it is a .aux or .pdf file etc and it exists! if it doesn't exist, it will appear in the list of dependents regardless)
  \@addtofilelist{#1}% if you want it to appear in \listfiles, not really necessary and latexmk doesn't use this
  \IfFileExists{#1}{}{\typeout{No file #1.}}% latexmk will find this message if #1 doesn't exist (yet)
}
\makeatother

\newcommand{\mytitle}{Scalable Metropolis--Hastings for Exact Bayesian Inference with Large Datasets}

\icmltitlerunning{\mytitle}

\usepackage{siunitx}
\usepackage{amsmath}
\usepackage{amssymb}
\usepackage{amsthm}
\usepackage{mathtools}
\usepackage{mathrsfs}
\usepackage{bbm}
\usepackage{booktabs}
\usepackage{multirow}
\usepackage{tabularx}
\usepackage{upgreek}

\newtheorem{theorem}{Theorem}[section]
\newtheorem{lemma}{Lemma}[section]
\newtheorem{proposition}{Proposition}[section]
\newtheorem{corollary}{Corollary}[section]

\newcommand{\Z}{\mathbb{Z}}

\newcommand{\R}{\mathbb{R}}

\DeclareMathOperator*{\argmin}{arg\,min}

\newcommand{\set}[1]{\{#1\}}

\newcommand{\abs}[1]{|#1|}
\newcommand{\Abs}[1]{\left|#1\right|}
\newcommand{\norm}[1]{\lVert#1\rVert}
\newcommand{\inner}[2]{\langle#1, #2\rangle}

\newcommand{\E}{\mathbb{E}}
\newcommand{\Var}{\mathrm{Var}}
\newcommand{\var}{\mathrm{var}}

\renewcommand{\P}{\mathbb{P}}
\newcommand{\iid}{\overset{\mathrm{iid}}{\sim}}

\newcommand{\Pto}[1]{\overset{#1}{\to}}

\newcommand{\ind}{\mathbb{I}}

\newcommand{\Normal}{\mathrm{Normal}}
\newcommand{\Bernoulli}{\mathrm{Bernoulli}}

\newcommand{\Poisson}{\mathrm{Poisson}}

\newcommand{\Categorical}{\mathrm{Categorical}}
\newcommand{\Leb}{\mathrm{Leb}}

\newcommand{\Student}{\mathrm{Student}}

\DeclarePairedDelimiterX{\infdivx}[2]{(}{)}{#1\;\delimsize\|\;#2}
\newcommand{\KL}{D_{\mathrm{KL}}\infdivx*}

\newcommand{\opnorm}[1]{\norm{#1}_{\mathrm{op}}}

\newcommand{\TV}[1]{\norm{#1}_{\mathrm{TV}}}

\newcommand{\Pgeneric}{P}
\newcommand{\Pfmh}{\Pgeneric_\mathrm{FMH}}
\newcommand{\Pmh}{\Pgeneric_\mathrm{MH}}
\newcommand{\Ptfmh}{\Pgeneric_\mathrm{TFMH}}
\newcommand{\Ptrue}{{\Pgeneric_0}}
\newcommand{\ptrue}{{p_0}}

\newcommand{\ageneric}{\alpha}
\newcommand{\afmh}{{\ageneric_{\mathrm{FMH}}}}
\newcommand{\amh}{\ageneric_{\mathrm{MH}}}
\newcommand{\atfmh}{\ageneric_{\mathrm{TFMH}}}
\newcommand{\asmh}[1]{\ageneric_{\mathrm{SMH}\text{-}{#1}}}
\newcommand{\smh}[1]{SMH-#1}

\newcommand{\rmh}{r_{\mathrm{MH}}}
\newcommand{\rfmh}{r_{\mathrm{FMH}}}
\newcommand{\rtfmh}{r_{\mathrm{TFMH}}}

\newcommand{\prop}{q}

\newcommand{\target}{\pi}
\newcommand{\approxtarget}{\widehat{\target}}
\newcommand{\basetarget}{\widetilde{\target}}

\newcommand{\pot}{U}
\newcommand{\approxpot}{\widehat{\pot}}
\newcommand{\potgradub}{\overline{\pot}}

\newcommand{\rpropmat}{A}
\newcommand{\rpropvec}{b}
\newcommand{\rpropcov}{C}
\newcommand{\idmat}{I_d}

\newcommand{\state}{\theta}
\newcommand{\nstate}{\state'}
\newcommand{\modestate}{\widehat{\state}}
\newcommand{\mapstate}{\state_{\mathrm{MAP}}}
\newcommand{\mlestate}{\state_{\mathrm{MLE}}}
\newcommand{\statespace}{\mathit{\Theta}}
\newcommand{\truestate}{\state_0}
\newcommand{\convstate}{\state^\ast}

\newcommand{\data}{y}
\newcommand{\Data}{Y}
\newcommand{\dataspace}{\mathcal{Y}}
\newcommand{\intensity}{\lambda}
\newcommand{\bound}{\overline{\intensity}}
\newcommand{\statebound}{\varphi}
\newcommand{\databound}{\psi}
\newcommand{\Databound}{\Psi}

\newcommand{\transrad}{R}

\newcommand{\Lobj}{\mathcal{L}}
\newcommand{\midx}{\beta}
\newcommand{\Gap}{\mathrm{Gap}}

\begin{document}

\twocolumn[
    \icmltitle{\mytitle}



\begin{icmlauthorlist}
\icmlauthor{Rob Cornish}{ox}
\icmlauthor{Paul Vanetti}{ox}
\icmlauthor{Alexandre Bouchard-C\^ot\'e}{ubc}
\icmlauthor{George Deligiannidis}{ox,ati}
\icmlauthor{Arnaud Doucet}{ox,ati}
\end{icmlauthorlist}

\icmlaffiliation{ox}{University of Oxford, Oxford, United Kingdom}
\icmlaffiliation{ubc}{University of British Columbia, Vancouver, Canada}
\icmlaffiliation{ati}{The Alan Turing Institute, London, United Kingdom}
\icmlcorrespondingauthor{Rob Cornish}{rcornish@robots.ox.ac.uk}

\icmlkeywords{Machine Learning, ICML, Markov Chain Monte Carlo, MCMC, Subsampling, Bayesian, Scalable, SMH, FMH}

    \vskip 0.3in
]

\printAffiliationsAndNotice{}

\begin{abstract}
Bayesian inference via standard Markov Chain Monte Carlo (MCMC) methods is too computationally intensive to handle large datasets, since the cost per step usually scales like $\Theta(n)$ in the number of data points $n$. We propose the \emph{Scalable Metropolis--Hastings} (SMH) kernel that exploits Gaussian concentration of the posterior to require processing on average only $O(1)$ or even $O(1/\sqrt{n})$ data points per step. This scheme is based on a combination of factorized acceptance probabilities, procedures for fast simulation of Bernoulli processes, and control variate ideas. Contrary to many MCMC subsampling schemes such as fixed step-size Stochastic Gradient Langevin Dynamics, our approach is exact insofar as the invariant distribution is the true posterior and not an approximation to it. We characterise the performance of our algorithm theoretically, and give realistic and verifiable conditions under which it is geometrically ergodic. This theory is borne out by empirical results that demonstrate overall performance benefits over standard Metropolis--Hastings and various subsampling algorithms.
\end{abstract}
\section{Introduction} \label{sec:introduction}

Bayesian inference is concerned with the posterior distribution $p(\state|\data_{1:n})$, where $\state \in \statespace = \R^d$ denotes parameters of interest and $\data_{1:n} = (\data_1, \cdots, \data_n) \in \dataspace^n$ are observed data. We assume the prior admits a Lebesgue density  $p(\state)$ and that the data are conditionally independent given $\state$ with likelihoods $p(\data_i|\state)$, which means
\[
    p(\state|\data_{1:n}) \propto p(\state) \prod_{i=1}^n p(\data_i|\state).
\]
In most cases of interest, $p(\state|\data_{1:n})$ does not admit a closed-form expression and so we must resort to a Markov Chain Monte Carlo (MCMC) approach. However, standard MCMC schemes can become very computationally expensive for large datasets. For example, the Metropolis--Hastings (MH) algorithm requires computing a likelihood ratio $p(\data_{1:n}|\state')/p(\data_{1:n}|\state)$ at each iteration. A direct implementation of this algorithm thus requires computational cost $\Theta(n)$ per step, which is prohibitive for large $n$.

\begin{figure}\label{fig:Averagenumberlikelihoodevaluations}
%\vskip 0.2in
\begin{center}
\centerline{\includegraphics[width=\columnwidth]{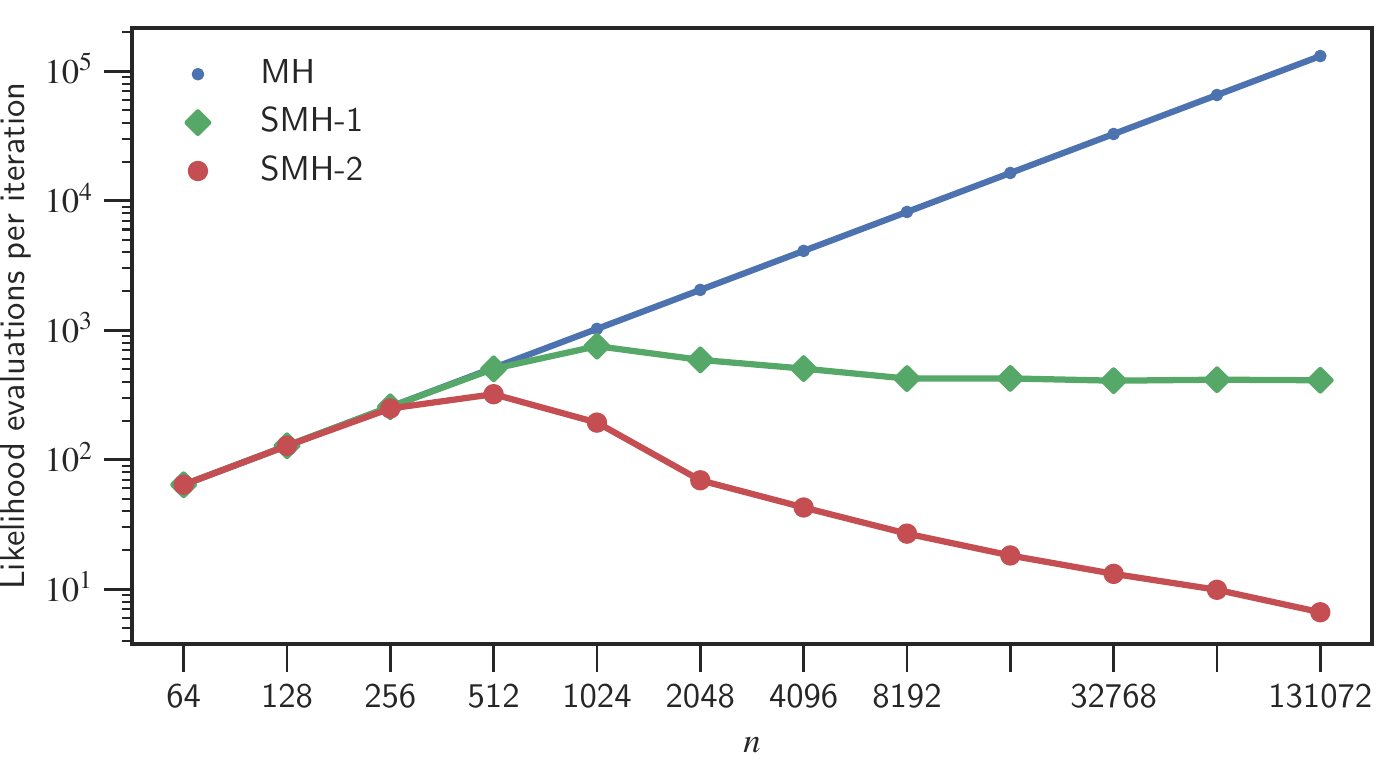}}
%\caption{Average number of likelihood evaluations per iteration as the number $n$ of data increases for MH algorithm and SMH using first-order (\smh{1}) and second-order (\smh{2}) proxy targeting a logistic regression posterior.}
\caption{Average number of likelihood evaluations per iteration required by SMH for a 10-dimensional logistic regression posterior as the number of data points $n$ increases. \smh{1} uses a first-order approximation to the target and \smh{2} a second-order one.}
\label{fig:likelihoods_per_iteration}
\end{center}
%\vskip -0.2in
\end{figure}

%to evaluate  involves evaluating the likelihood of involve repeatedly transitioning from a current point $\state$ to a new point $\nstate$ such that $\nstate \sim \target$ if $\state \sim \target$. The celebrated Metropolis-Hastings algorithm (MH) \cite{hastings1970monte} is one such method which has seen widespread use due to its simplicity and generality. MH works by proposing a new point, and then accepting or rejecting this proposal based on a ratio of density values.

%However, while broadly applicable, MH can become expensive to use for large datasets. In particular, computing the density ratio requires $n$ likelihood evaluations, so that the computational cost of a single MH iteration grows linearly in the number of data points. This cost can very quickly become intractable when MH is applied in a Bayesian big data setting.

%MH requires the user specify a proposal
%$\prop(\state, \nstate)$ such that $\prop(\state, \cdot)$ is a Lebesgue density from which
%we can sample for each $\state \in \statespace$. At each step, MH proposes $\nstate \sim
%\prop(\state, \cdot)$ and accepts with probability
%\begin{equation} \label{eq:mh_acc_prob}
%  \amh(\state, \nstate) := 1 \wedge \frac{\target(\nstate) \prop(\nstate, \state)}{\target(\state) \prop(\state, \nstate)}.
%\end{equation}
%That this procedure preserves $\target$ as its stationary distribution is a standard
%result.

Many ideas for mitigating this cost have been suggested; see \citet{bardenet2017markov} for a recent review. Broadly speaking these approaches are distinguished by whether they exactly preserve the true posterior as the invariant distribution of the Markov chain produced. Approximate methods that have been proposed include \emph{divide-and-conquer} schemes, which run parallel MCMC chains on a partition of the data \cite{neiswanger2013asymptotically, scott2016bayes}. Other approaches replace the likelihood ratio in MH with an approximation computed from a subsample of observations. The error introduced can be controlled heuristically using central limit theorem approximations \cite{korattikara2013austerity} or rigorously via concentration inequalities \cite{bardenet2014towards,bardenet2017markov,quiroz2014speeding}. Another popular class of schemes is based on Stochastic Gradient Langevin Dynamics (SGLD) \cite{welling2011bayesian,dubey2016variance,baker2017control,brosse2018promises,chatterji2018theory}, which is a time-discretized Langevin dynamics where the gradient of the log-likelihood is approximated by subsampling. SGLD is usually implemented using a fixed step-size discretization, which does not exactly preserve the posterior distribution.
%A decreasing step-size version can provide consistent estimators of expectations w.r.t.\ $\pi$ but at the cost of reduced convergence rate compared to standard MCMC \cite{teh2016consistency}. 
Finally, \citet{quiroz2016block} and \citet{dang2017hamiltonian} propose schemes that do not preserve the posterior exactly, but yield consistent estimates of posterior expectations after an importance sampling correction.

In addition to these approximate methods, several MCMC methods exist that do preserve the target as invariant distribution while only requiring access to a subset of the data at each iteration. However, various restrictions of these approaches have so far limited their widespread use.
%Stochastic Gradient Langevin Dynamics (SGLD) \cite{welling2011bayesian} suppresses the MH acceptance step with a carefully designed update rule that requires computing only a minibatch of likelihood terms per iteration.
Firefly Monte Carlo \cite{maclaurin2014firefly} considers an extended target that can be evaluated using a subset of the data at each iteration, but requires the user specify global lower bounds to the likelihood factors that can be difficult to derive. It is as yet also unclear what the convergence properties of this scheme are. Delayed acceptance schemes have been proposed based on a factorized version of the MH acceptance probability \cite{banterle2015accelerating} and on a random subsample of the data \cite{payne2018two}. These methods allow rejecting a proposal without computing every likelihood term, but still require evaluating each term in order to accept. \citet{quiroz2018speeding} combine the latter with the approximate subsampling approach of \citet{quiroz2014speeding} to mitigate this problem. Finally, various non-reversible continuous-time MCMC schemes based on Piecewise Deterministic Markov Processes have been proposed which, when applied to large-scale datasets \cite{bouchard2015bouncy,bierkens2016zigzag}, only require evaluating the gradient of the log-likelihood for a subset of the data. However, these schemes can be difficult to understand theoretically, falling outside the scope of existing geometric ergodicity results, and can be challenging to implement.
%Both cases simulate a non-reversible Markov chain in continuous time with the correct invariant distribution. Under certain conditions, running these schemes requires updating only a subset of their state at a time, which in effect constitutes subsampling. This has generated considerable interest in exploring continuous time and non-reversible ideas as a possible means for scaling MCMC to the big data regime. However, these approaches can be quite difficult both to analyse theoretically and to implement in practice.

In this paper we present a novel MH-type subsampling scheme that exactly preserves the posterior as the invariant distribution while still enjoying attractive theoretical properties and being straightforward to implement and tune. We make use of a combination of a factorized MH acceptance probability \cite{ceperley1995path,christen2005markov,banterle2015accelerating,michel2017clock,vanetti2017piecewise} and fast methods for sampling non-homogeneous Bernoulli processes \cite{shanthikumar1985discrete,devroye1986,fukui2009order,michel2017clock,vanetti2017piecewise} to allow iterating without computing every likelihood factor. The combination of these ideas has proven useful for some physics models \cite{michel2017clock}, but a na\"ive application is not efficient for large-scale Bayesian inference. Our contribution here is an MH-style MCMC kernel that realises the potential computational benefits of this method in the Bayesian setting. We refer to this kernel as \emph{Scalable Metropolis-Hastings} (SMH) and, in addition to empirical results, provide a rigorous theoretical analysis of its behaviour under realistic and verifiable assumptions. In particular, we show SMH requires on average only $O(1)$ or even $O(1/\sqrt{n})$ cost per step as illustrated in Figure \ref{fig:Averagenumberlikelihoodevaluations}, has a non-vanishing average acceptance probability in the stationary regime, and is geometrically ergodic under mild conditions.

Key to our approach is the use of \emph{control variate} ideas, which allow us to exploit the concentration around the mode frequently observed for posterior distributions with large datasets. Control variate ideas based on posterior concentration have been used successfully for large-scale Bayesian analysis in numerous recent contributions \cite{dubey2016variance,bardenet2017markov,baker2017control,brosse2018promises,bierkens2016zigzag,chatterji2018theory,quiroz2014speeding}. In our setting, this may be understood as making use of a computationally cheap approximation of the posterior.

%The rest of this work is organized as follows. In Section \ref{sec:fmh} we review the factorized MH method and show how to implement this using procedures for fast simulation of Bernoulli random variables. We also provide original sufficient conditions for ensuring the method is geometrically ergodic. In Section \ref{sec:big-data} we show how this idea can be exploited in the Bayesian big data setting, detailing our proposed factorization scheme and giving realistic and verifiable conditions under which it can be expected to yield a benefit. Finally in Section \ref{sec:experiments} we show empirically that the predicted computational gains are indeed achieved for two real-world models.

The Supplement contains all our proofs as well as a guide to our notation in Section \ref{sec:notation}.
\section{Factorised Metropolis-Hastings} \label{sec:fmh}

We first review the use of a factorised acceptance probability inside an MH-style algorithm. For now we assume a generic target $\target(\state)$ before specialising to the Bayesian setting below.

\subsection{Transition Kernel}
Assume our target $\pi(\state)$ and proposal $q(\state,\state')$ factorise like
\[
  \target(\state) \propto \prod_{i=1}^m \target_i(\state) \quad\quad \prop(\state, \nstate) \propto \prod_{i=1}^m \prop_i(\state, \nstate)
\]
for some $m \geq 1$ and some choice of non-negative functions $\target_i$ and $\prop_i$. These factors are not themselves required to be integrable; for instance, we may take any $\target_i, \prop_i \equiv 1$. Define the \emph{Factorised Metropolis-Hastings} (FMH) kernel
\begin{multline} \label{eq:fmh-kernel}
  \Pfmh(\state, A)
    := \left(1 - \int \prop(\state, \nstate) \afmh(\state, \nstate) d\nstate\right) \ind_A(\state) \\
      + \int_A \prop(\state, \nstate) \afmh(\state, \nstate) d\nstate,
\end{multline}
where $\state \in \statespace$, $A \subseteq \statespace$ is measurable, and the FMH acceptance probability is defined
\begin{equation} \label{eq:fmh_acc_prob}
  \afmh(\state, \nstate) %:=  \prod_{i=1}^m \afmh_{i}(\state, \nstate)
  := \prod_{i=1}^m \underbrace{1 \wedge \frac{\target_i(\nstate) \prop_i(\nstate, \state)}{\target_i(\state) \prop_i(\state, \nstate)}}_{=:\afmh_i(\state, \nstate)}.
\end{equation}
It is straightforward and well-known that $\Pfmh$ is $\target$-reversible; see Section \ref{SUPP:fmh-reversible} in the Supplement for a proof. Factorised acceptance probabilities have appeared numerous times in the literature and date back at least to \cite{ceperley1995path}. The MH acceptance probability $\amh$ and kernel $\Pmh$ correspond to $\afmh$ and $\Pfmh$ when $m = 1$.

\subsection{Poisson Subsampling Implementation} \label{sec:subsampling-via-ppp}

The acceptance step of $\Pfmh$ can be implemented by sampling directly $m$ independent Bernoulli trials with success probability $1 - \afmh_{i}$, and returning $\nstate$ if every trial is a failure. Since we can reject $\state'$ as soon as a single success occurs, this allows us potentially to reject $\nstate$ without computing each factor at each iteration \cite{christen2005markov,banterle2015accelerating}.

However, although this can lead to efficiency gains in some contexts, it remains of limited applicability for Bayesian inference with large datasets since we are still forced to compute every factor whenever we accept a proposal. It was realized independently by \citet{michel2017clock} and \citet{vanetti2017piecewise} that if one has access to lower bounds on $\afmh_{i}(\state, \nstate)$, hence to an upper bound on $1-\afmh_{i}(\state, \nstate)$, then techniques for fast simulation of Bernoulli random variables can be used that potentially avoid this problem. One such technique is given by the discrete-time thinning algorithms introduced in \cite{shanthikumar1985discrete}; see also \citep[Chapter~VI~Sections~3.3-3.4]{devroye1986}. This is used in \cite{michel2017clock}.

We use here an original variation of a scheme developed in \cite{fukui2009order}. Denote
\[
  \intensity_i(\state, \nstate) := -\log \afmh_i(\state, \nstate),
\]
and assume we have the bounds
\begin{equation} \label{eq:lambda-bounds}
  \intensity_i(\state, \nstate) \leq \statebound(\state, \nstate) \databound_i
  := \bound_i(\state, \nstate)
\end{equation}
for nonnegative $\statebound, \databound_i$. This condition holds for a variety of statistical models: for instance, if $\target_i$ is log-Lipschitz and $\prop$ is symmetric with (say) $\prop_i = \prop^{1/m}$, then
\begin{equation} \label{eq:lipschitz-bounds}
    \intensity_i(\state, \nstate) \leq K_i \norm{\state - \nstate}.
\end{equation}
This case illustrates that \eqref{eq:lambda-bounds} is usually a \emph{local} constraint on the target and therefore not as strenuous as the global lower-bounds required by Firefly \cite{maclaurin2014firefly}. We exploit this to provide a methodology for producing $\statebound$ and $\databound$ mechanically when we consider Bayesian targets in Section \ref{sec:big-data}. Letting $\bound(\state, \nstate) := \sum_{i=1}^m \bound_i(\state, \nstate)$, it follows that if
\begin{itemize}[label={}, leftmargin=*]
  \item $N \sim \Poisson\left(\bound(\state, \nstate)\right)$
  \item $X_1, \cdots, X_N \iid \Categorical((\bound_i(\state,
    \nstate) / \bound(\state, \nstate))_{1\leq i \leq m})$
  \item $B_j \sim \Bernoulli(\intensity_{X_j}(\state, \nstate) /
   \bound_{X_j}(\state, \nstate))$ independently for $1 \leq j \leq N$
\end{itemize}
then $\P(B = 0) = \afmh(\state, \nstate)$ where $B = \sum_{j=1}^N B_j$ (and $B = 0$ if $N = 0$). See Proposition
\ref{prop:poisson-subsampling} in the Supplement for a proof. These steps may be interpreted as sampling a discrete Poisson point process with intensity $\intensity_i(\state, \nstate)$ on $i \in \set{1, \cdots, m}$ via thinning \cite{devroye1986}. Thus, to perform the FMH acceptance step, we can simulate these $B_j$ and check whether each is $0$.

We may exploit \eqref{eq:lambda-bounds} to sample each $X_j$ and $B_j$ in $O(1)$ time per MCMC step as $m \to \infty$ after paying some once-off setup costs. Note that
\begin{equation} \label{eq:upper-bounds-sum}
  \bound(\state, \nstate) = \statebound(\state, \nstate) \sum_{i=1}^m \databound_i, 
\end{equation}
so that we may compute $\bound(\state, \nstate)$ in $O(1)$ time per iteration by simply
evaluating $\statebound(\state, \nstate)$ if we pre-compute $\sum_{i=1}^m \databound_i$ ahead of
our run. This incurs a one-time cost of $\Theta(m)$, but assuming our run is long enough
this will be negligible overall. Similarly, note that
\[
  \frac{\bound_i(\state, \nstate)}{\bound(\state, \nstate)}
    = \frac{\databound_i}{\sum_{j=1}^m \databound_j},
\]
so that $\Categorical((\bound_i(\state, \nstate) / 
\bound(\state, \nstate))_{1\leq i \leq m})$ does not depend on
$\state, \nstate$. Thus, we can sample each $X_i$ in $O(1)$ time using Walker's alias
method \cite{walker1977efficient,kronmal1979alias} having paid another once-off $\Theta(m)$ cost.

Algorithm \ref{alg:poisson-sampling} shows how to implement $\Pfmh$ using this approach. Observe that if $N < m$ we are guaranteed not to evaluate every target factor even if we accept the proposal $\nstate$. Of course, since $N$ is random, in general it is not obvious that $N \ll m$ will necessarily hold on average, and indeed this will not be so for a na\"ive factorisation. We show in Section \ref{sec:big-data} how to use Algorithm \ref{alg:poisson-sampling} as the basis of an efficient subsampling method for Bayesian inference.

In many cases we will not have bounds of the form \eqref{eq:lambda-bounds} for every factor. However, Algorithm \ref{alg:poisson-sampling} can still be useful provided the computational cost for
computing these extra factors is $O(1)$. In this case we can directly simulate a Bernoulli trial for each additional factor, which by assumption does not change the asymptotic complexity of this method.

\begin{algorithm}[tb]
    \caption{Efficient implementation of the FMH kernel. $\mathrm{Setup()}$ is called once
    prior to starting the MCMC run.}
   \label{alg:poisson-sampling}
\begin{algorithmic}
  \FUNCTION{$\mathrm{Setup()}$}
    \STATE $\Databound \gets \sum_{i=1}^m \databound_i$
    \STATE $\tau \gets \mathrm{AliasTable}((\databound_i/\Databound)_{1 \leq i \leq m})$
  \ENDFUNCTION
  \STATE
  \FUNCTION{$\mathrm{FmhKernel}(\state)$}
    \STATE $\nstate \sim \prop(\state, \cdot)$
    \STATE $N \sim \Poisson(\statebound(\state, \nstate) \Databound)$
    \FOR{$j \in 1,...,N$}
      \STATE $X_j \sim \tau$
      \STATE $B_j \sim \Bernoulli(\intensity_{X_j}(\state, \nstate) / \bound_{X_j}(\state, \nstate))$
      \IF{$B_j = 1$}
      \STATE {\bf return} $\state$
      \ENDIF
    \ENDFOR
    \STATE {\bf return} $\nstate$
  \ENDFUNCTION
\end{algorithmic}
\end{algorithm}

\subsection{Geometric Ergodicity} \label{sec:geom-erg}

We consider now the theoretical implications of using $\Pfmh$ rather than $\Pmh$. We refer the reader to Section \ref{SUPP:sec:ergodic-props} in the Supplement for a review of the relevant definitions and theory of Markov chains. It is straightforward to show and well-known that the following holds.
\begin{proposition} \label{prop:fmh-lowers-accprob}
    For all $\state, \nstate \in \statespace$, $\afmh(\state, \nstate) \leq \amh(\state, \nstate)$.
\end{proposition}
See Section \ref{sec:acc-probs} in the Supplement for a proof. As such, we do not expect FMH to enjoy better convergence properties than MH. Indeed, Proposition \ref{prop:fmh-lowers-accprob} immediately entails that FMH produces ergodic averages of higher asymptotic variance than standard MH \cite{peskun1973optimum,tierney1998note}. Moreover $\Pfmh$ can fail to be geometrically ergodic even when $\Pmh$ is, as noticed by \citet{banterle2015accelerating}. Geometric ergodicity is a desirable property of MCMC algorithms because it ensures the central limit theorem holds for some ergodic averages \citep[Corollary 2.1]{roberts1997geometric}. The central limit theorem in turn is the foundation of  principled stopping criteria based on Monte Carlo standard errors \cite{jones_honest_2001}.

To address the fact that $\Pfmh$ might not be geometrically ergodic, we introduce the \emph{Truncated FMH} (TFMH) kernel $\Ptfmh$ which is obtained by simply replacing in \eqref{eq:fmh-kernel} the term $\afmh(\state, \nstate)$ with the acceptance probability
\begin{equation} \label{eq:tfmh-acc-prob}
  \atfmh(\state, \nstate) := \begin{cases}
    \afmh(\state, \nstate), & \bound(\state, \nstate) < \transrad \\
    \amh(\state, \nstate), & \text{otherwise},
  \end{cases}
\end{equation}
for some choice of $\transrad \in [0, \infty]$. Observe that FMH is a special case of TFMH with $\transrad = \infty$. When $\bound(\state, \nstate)$ is symmetric in $\state$ and $\nstate$, Proposition \ref{prop:ptfmh-reversible} in the Supplement shows that $\Ptfmh$ is still $\target$-reversible. The following theorem shows that under mild conditions TFMH inherits the desirable convergence properties of MH.

\begin{theorem}
  If $\Pmh$ is $\varphi$-irreducible, aperiodic, and geometrically ergodic, then $\Ptfmh$ is too if
  \begin{equation} \label{eq:geom-erg-cond}
    \delta := \inf_{\bound(\state, \nstate) < \transrad} \afmh(\state, \nstate) \vee \afmh(\nstate, \state) > 0.
  \end{equation}
  In this case, $\mathrm{Gap}(\Pfmh) \geq \delta \mathrm{Gap}(\Pmh)$, and for $f \in L^2(\target)$
  \[
    \var(f, \Ptfmh) \leq (\delta^{-1} - 1) \var(f, \target) + \delta^{-1} \var(f, \Pmh).
  \]
\end{theorem}

Here $\Gap(\Pgeneric)$ denotes the spectral gap and $\var(f, \Pgeneric)$ the asymptotic variance of the ergodic averages of $f$. See Section \ref{SUPP:sec:ergodic-props} in the Supplement for full definitions and a proof. Proposition \ref{prop:afmh-amh-relation} in the Supplement shows that $\afmh(\state, \nstate) \vee \afmh(\nstate, \state) = \afmh(\state, \nstate) / \amh(\state,\nstate)$, and hence \eqref{eq:geom-erg-cond} quantifies the worst-case cost we pay for using the FMH acceptance probability rather than the MH one. The condition \eqref{eq:geom-erg-cond} is easily seen to hold in the common case that each $\target_i$ is bounded away from $0$ and $\infty$ on $\set{\state, \nstate \in \statespace \mid \bound(\state, \nstate) < \transrad}$, which is a fairly weak requirement when $R<\infty$.

Recall from the previous section that $\Pfmh$ requires computing $N \sim \Poisson(\bound(\state, \nstate))$ factors for a given $\state, \nstate$. In this way, TFMH yields the additional benefit of controlling the maximum expected number of factors we will need to compute via the choice of $\transrad$. An obvious choice is to take $\transrad = m$, which ensures we will not compute more factors for FMH than for MH on average. Thus, overall, TFMH yields the computational benefits of $\afmh$ when our bounds \eqref{eq:lambda-bounds} are tight (usually near the mode), and otherwise falls back to MH as a default (usually in the tails).

\section{FMH for Bayesian Big Data} \label{sec:big-data}

We now consider the specific application of FMH to the problem of Bayesian inference for large datasets, where $\target(\state) \propto p(\state) \prod_{i=1}^n p(\data_i|\state)$. It is frequently observed that such targets concentrate at a rate $1/\sqrt{n}$ around the mode as $n \to \infty$, in what is sometimes referred to as the Bernstein-von Mises phenomenon. We describe here how to leverage this phenomenon to devise an effective subsampling algorithm based on Algorithm \ref{alg:poisson-sampling}. Our approach is based on control variate ideas similar to \citet{dubey2016variance,bardenet2017markov,bierkens2016zigzag,baker2017control,chatterji2018theory,quiroz2014speeding}. We emphasise that all these techniques also rely on a posterior concentration assumption but none of them only requires processing $O(1/\sqrt{n})$ data points per iteration as we do.

To see why this approach is needed, observe that the most natural factorisations of the posterior have $m \asymp n$. This introduces a major pitfall: each new factor introduced can only lower the value of $\afmh(\state, \nstate)$, which in the aggregate can therefore mean $\afmh(\state, \nstate) \to 0$ as $n \to \infty$. 
%Unless are careful to control the behaviour of the acceptance probability, we can easily diminish the statistical efficiency of the Markov chain produced to the extent that any computational benefits obtained by subsampling are undone.

Consider heuristically a na\"ive application of Algorithm \ref{alg:poisson-sampling} to $\pi$. Assuming a flat prior for simplicity, the obvious factorisation takes $m = n$ and each $\target_i(\state) = p(\data_i|\state)$. Suppose the likelihoods are log-Lipschitz and that we use the bounds \eqref{eq:lipschitz-bounds} derived above. For smooth likelihoods, if the Lipschitz constants $K_i$ are chosen minimally, these bounds will be tight in the limit as $\norm{\state - \nstate} \to 0$. Consequently, if we scale $\norm{\state - \nstate}$ as $1/\sqrt{n}$ to match the concentration of the target, then $\afmh(\state, \nstate) \asymp \exp(-\bound(\state, \nstate)) \to 0$ since
\[
    \bound(\state, \nstate) = \underbrace{\norm{\state - \nstate}}_{=\Theta(1/\sqrt{n})} \underbrace{\sum_{i=1}^n K_i}_{=\Theta(n)} = \Theta(\sqrt{n}).
\]
Recall that Algorithm \ref{alg:poisson-sampling} requires the computation of at most $N \sim \Poisson(\bound(\state, \nstate))$ factors, and hence in this case we do obtain a reduced expected cost per iteration of $\Theta(\sqrt{n})$ as opposed to $\Theta(n)$. Nevertheless, we found empirically that the increased asymptotic variance produced by the decaying acceptance probability entails an overall loss of performance compared with standard MH. We could consider using a smaller stepsize such as $\norm{\state - \nstate} = O(1/n)$ which would give a stable acceptance probability, but then our proposal would not match the $1/\sqrt{n}$ concentration of the posterior. We again found this increases the asymptotic variance to the extent that it negates the benefits of subsampling overall.

\subsection{Scalable Metropolis--Hastings} \label{sec:smh}

Our approach is based on controlling $\bound(\state, \nstate)$, which ensures both a low computational cost and a large acceptance probability. We assume an initial factorisation
\begin{equation} \label{eq:base-factorisation}
    \target(\state) \propto p(\state) \prod_{i=1}^n p(\data_i|\state) \propto \prod_{i=1}^{m} \widetilde{\target}_i(\state)
\end{equation}
for some $m$ (not necessarily equal to $n$) and $\widetilde{\target}_i$ (e.g. using directly the factorisation of prior and likelihoods). Let
\[
  \pot_i(\state) := -\log \widetilde{\target}_i(\state) \quad \quad U(\state) := \sum_{i=1}^{m} \pot_i(\state).
\]
We choose some fixed $\modestate \in \statespace$ not depending on $i$ that is near the mode of $\target$ like \citet{dubey2016variance,bardenet2017markov, bierkens2016zigzag, baker2017control, chatterji2018theory,quiroz2014speeding}. Assuming sufficient differentiability, we then approximate $\pot_i$ with a $k$-th order Taylor expansion around $\modestate$, which we denote by
\[
  \approxpot_{k,i}(\state) \approx \pot_i(\state).
\]
We also define
\[
  \approxtarget_{k,i}(\state) := \exp(-\approxpot_{k,i}(\state)) \approx \widetilde{\target}_i(\state).
\]
In practice we are exclusively interested in the cases $k = 1$ and $k = 2$, which correspond to first and second-order approximations respectively. Explicitly, in these cases
\begin{eqnarray*}
    \approxpot_{1,i}(\state) &=& \pot(\modestate) + \nabla \pot_i(\modestate)^\top (\state - \modestate), \\
    \approxpot_{2,i}(\state) &=& \approxpot_{1,i}(\state) + \frac{1}{2}  (\state - \modestate)^\top \nabla^2 \pot_i(\modestate) (\state - \modestate),
\end{eqnarray*}
where $\nabla$ denotes the gradient and $\nabla^2$ the Hessian. Letting
\[
    \approxpot_k(\state) := \sum_{i=1}^m \approxpot_{k,i}(\state)
    \quad\quad
    \approxtarget_k(\state) := \exp(-\approxpot_k(\state)),
\]
additivity of the Taylor expansion further yields
\begin{eqnarray}
    \approxpot_{1}(\state) &=& \pot(\modestate) + \nabla \pot(\modestate)^\top (\state - \modestate) \notag \\
    \approxpot_{2}(\state) &=& \approxpot_{1}(\state) + \frac{1}{2}  (\state - \modestate)^\top \nabla^2 \pot(\modestate) (\state - \modestate). \label{eq:2nd-order-approx}
\end{eqnarray}
Thus when $\nabla^2 \pot(\modestate) \succ 0$ (i.e. symmetric positive-definite), $\approxtarget_2(\state)$ is seen to be a Gaussian approximation to $\target$ around the (approximate) mode $\modestate$.

We use the $\approxtarget_{k,i}$ to define the \emph{Scalable Metropolis-Hastings} (SMH or \smh{$k$}) acceptance probability
\begin{equation} \label{eq:afmh}
    \asmh{k}(\state, \nstate) := \left(1 \wedge \frac{\approxtarget_k(\nstate) \prop(\nstate, \state)}{\approxtarget_k(\state) \prop(\state, \nstate)} \right) \prod_{i=1}^{m} 1 \wedge \frac{\widetilde{\target}_i(\nstate) \approxtarget_{k,i}(\state)}{\approxtarget_{k,i}(\nstate) \widetilde{\target}_i(\state)}.
\end{equation}
Note that \smh{$k$} is a special case of FMH with $m+1$ factors given by
\begin{equation} \label{eq:cv-factorisation}
  \target = 
  \underbrace{\approxtarget_k}_{= \target_{m+1}} \prod_{i=1}^m \underbrace{\frac{\widetilde{\target}_i}{\approxtarget_{k,i}}}_{= \target_i}
  \quad\quad
  \prop = \underbrace{\prop}_{= \prop_{m+1}} \prod_{i=1}^m \underbrace{1}_{= \prop_i}
\end{equation}
and hence defines a valid acceptance probability. (Note that $\approxtarget_1$ is not integrable, but recall this is not required of FMH factors.) We could consider any factorisation of $\prop$, but we will not make use of this generality.

$\approxtarget_k(\state)$ can be computed in constant time after precomputing the relevant partial derivatives at $\modestate$ before our MCMC run. This allows us to deal with $1 \wedge \approxtarget_k(\nstate) \prop(\nstate, \state)/\approxtarget_k(\state) \prop(\state, \nstate)$ by directly simulating a Bernoulli trial with this value as its success probability. For the remaining factors we have
\begin{eqnarray*}
  \intensity_i(\state, \nstate)
    &=& -\log \left(1 \wedge \frac{\widetilde \target_i(\nstate) \approxtarget_{k,i}(\state)}{\widetilde \target_i(\state) \approxtarget_{k,i}(\nstate)}\right).
\end{eqnarray*}
We can obtain a bound of the form \eqref{eq:lambda-bounds} provided $\pot_i$ is $(k+1)$-times
continuously differentiable. In this case, if we can find constants
\begin{equation} \label{eq:grad-upper-bound}
  \potgradub_{k+1, i}
    \geq \sup_{\substack{\state \in \statespace \\ \abs{\midx} = k+1}}
      \abs{\partial^\midx \pot_i(\state)},
\end{equation}
(here $\midx$ is multi-index notation; see Section \ref{sec:notation} of the Supplement) it follows that
\begin{equation} \label{eq:smh-bound}
    \bound(\state, \nstate) := \underbrace{(\norm{\state - \modestate}_1^{k+1} + \norm{\nstate - \modestate}_1^{k+1})}_{= \statebound(\state, \nstate)} \sum_{i=1}^m \underbrace{\frac{\potgradub_{k+1, i}}{(k+1)!}}_{= \databound_i}
\end{equation}
defines an upper bound of the required form \eqref{eq:upper-bounds-sum}. See Proposition \ref{SUPP:prop:upper-bounds} in the Supplement for a derivation. Observe this is symmetric in $\state$ and $\nstate$ and therefore can be used to define a truncated version of SMH as described in Section \ref{sec:geom-erg}.

Although we concentrate on Taylor expansions here, other choices of $\approxtarget_i$ may be useful. For instance, it may be possible to make $\basetarget_i/\approxtarget_i$ log-Lipschitz or log-concave and obtain better bounds. However, Taylor expansions have the advantage of generality and \eqref{eq:smh-bound} is sufficiently tight for us.

Heuristically, if the posterior concentrates like $1/\sqrt{n}$, if we scale our proposal like $1/\sqrt{n}$, and if $\modestate$ is not too far (specifically $O(1/\sqrt{n})$) from the mode, then both $\norm{\state - \modestate}$ and $\norm{\nstate - \modestate}$ will be $O(1/\sqrt{n})$, and $\statebound(\state, \nstate)$ will be $O(n^{-(k+1)/2})$. If moreover $m \asymp n$, then the summation will be $O(n)$ and hence overall $\bound(\state, \nstate) = O(n^{(1-k)/2})$. When $k=1$ this is $O(1)$ and when $k=2$ this is $O(1/\sqrt{n})$, which entails a substantial improvement over the na\"ive approach. In particular, we expect stable acceptance probabilities in both cases, constant expected cost in $n$ for $k = 1$, and indeed $O(1/\sqrt{n})$ \emph{decreasing} cost for $k = 2$. We make this argument rigorous in Theorem \ref{thm:bound-asymptotics} below.

Beyond what is already needed for MH, $\potgradub_{k+1,i}$ and $\modestate$ are all the user must provide for our method. In practice neither of these seems problematic in typical settings. We have found deriving $\potgradub_{k+1,i}$ to be a fairly mechanical procedure, and give examples for two models in Section \ref{sec:experiments}. Likewise, while computing $\modestate$ does entail some cost, we have found that standard gradient descent finds an adequate result in time negligible compared with the full MCMC run.

\subsection{Choice of Proposal} \label{sec:prop-choice}

We now consider the choice of proposal $\prop$ and its implications for the acceptance probability. As mentioned, it is necessary to ensure that, roughly speaking, $\norm{\state - \nstate} = O(n^{-1/2})$ to match the concentration of the target. In this section we describe heuristically how to ensure this. Theorem \ref{thm:bound-asymptotics} below and Section \ref{sec:prop-scale} in the Supplement give precise statements of what is required.

Two main classes of $\prop$ are of interest to us. When $\prop$ is symmetric, \eqref{eq:afmh} simplifies to
\begin{equation} \label{eq:afmh-symmetric}
 \asmh{k}(\state, \nstate)
    = \left(1 \wedge \frac{\approxtarget_k(\nstate)}{\approxtarget_k(\state)} \right)
        \prod_{i=1}^{m} 1 \wedge \frac{\widetilde{\target}_i(\nstate)
          \approxtarget_{k,i}(\state)}{\widetilde{\target}_i(\state) \approxtarget_{k,i}(\nstate)}.
\end{equation}
We can realise this with the correct scaling with for example
\begin{equation} \label{eq:scaled-random-walk-prop}
    \prop(\state, \nstate) = \Normal(\nstate \mid \state, \frac{\sigma^2}{n} \idmat),
\end{equation}
where $\sigma > 0$ is fixed in $n$. Alternatively, we can more closely match the covariance of our proposal to the covariance of our target with
\begin{equation}\label{proposalpreconditionedRW}
    \prop(\state, \nstate) = \Normal(\nstate \mid \state, \sigma^2 [\nabla^2 \pot(\modestate)]^{-1}).
\end{equation}
Under usual circumstances $[\nabla^2 \pot(\modestate)]^{-1}$ is approximately (since in general this will include a non-flat prior term) proportional to the inverse observed information matrix, and hence the correct $O(n^{-1/2})$ scaling is achieved automatically. See Section \ref{sec:prop-scale} in the Supplement for more details.

We can improve somewhat on a symmetric proposal if we choose $\prop$ to be \emph{$\approxtarget_k$-reversible} in the sense that
\[
    \approxtarget_k(\state) \prop(\state, \nstate) = \approxtarget_k(\nstate) \prop(\nstate, \state)
\]
for all $\state, \nstate$; see, e.g., \cite{tierney1994,neal1999regression,kamatani2014efficient}. In this case we obtain
\[
   \asmh{k}(\state, \nstate) =
        \prod_{i=1}^{m} 1 \wedge \frac{\widetilde{\target}_i(\nstate)
          \approxtarget_{k,i}(\state)}{\widetilde{\target}_i(\state) \approxtarget_{k,i}(\nstate)}.
\]
Note that using a $\approxtarget_k$-reversible proposal allows us to drop the first term in \eqref{eq:afmh-symmetric}, and hence obtain a higher acceptance probability for the same $\state, \nstate$. Moreover, when $k = 2$, we see from \eqref{eq:2nd-order-approx} that a $\approxtarget_k$-reversible proposal corresponds to an MCMC kernel that targets a Gaussian approximation to $\target$, and may therefore be more suited to the geometry of $\target$ than a symmetric one.

We now consider how to produce $\approxtarget_k$-reversible proposals. For $\prop$ of the form
\[
    \prop(\state, \nstate) = \Normal(\nstate \mid \rpropmat \state + \rpropvec, \rpropcov)
\]
where $\rpropmat, \rpropcov \in \R^{d\times d}$ with $\rpropcov \succ 0$ and $\rpropvec \in \R^d$, Theorem \ref{SUPP:prop:reversible-prop-conds} in the Supplement gives necessary and sufficient conditions for $\approxtarget_1$ and $\approxtarget_2$-reversibility. Specific useful choices that satisfy these conditions and ensure the correct scaling are then as follows. For $\approxtarget_1$ we can use for example
\begin{equation} \label{eq:scaled-first-order-prop}
  \rpropmat = \idmat
  \quad\quad
  \rpropvec = - \frac{\sigma}{2n} \nabla U(\modestate)
  \quad\quad
  \rpropcov = \frac{\sigma}{n} \idmat
\end{equation}
for some $\sigma > 0$, where $\idmat \in \R^{d\times d}$ is the identity matrix. For $\approxtarget_2$, assuming $\nabla^2 \pot(\modestate) \succ 0$ (which will hold if $\modestate$ is sufficiently close to the mode), we can use a variation of the \textit{preconditioned-Crank Nicholson} proposal (pCN) \cite{neal1999regression} defined by taking
\begin{eqnarray*}
  \rpropmat = \sqrt{\rho} \idmat
  \quad\quad
  \rpropcov = (1-\rho) [\nabla^2 \pot(\modestate)]^{-1} \\
  \rpropvec = (1 - \sqrt{\rho}) (\modestate - [\nabla^2 \pot(\modestate)]^{-1} \nabla \pot(\modestate))
\end{eqnarray*}
where $\rho \in [0, 1)$. When $\rho = 0$ this corresponds to an independent Gaussian proposal: $\nstate \sim \approxtarget_2$. Note that this can be re-interpreted as the exact discretization 
of an Hamiltonian dynamics for the Gaussian $\approxtarget_2$.
\subsection{Performance}

We now show rigorously that SMH addresses the issues of a naive approach and entails an overall performance benefit. In our setup we assume some unknown data-generating distribution $\Ptrue$, with data $\Data_i \iid \Ptrue$. We denote the (random) targets by $\target^{(n)}(\state) := p(\state|\Data_{1:n})$, for which we assume a factorisation \eqref{eq:base-factorisation} involving $m^{(n)}$ terms. We denote the mode of $\target^{(n)}$ by $\mapstate^{(n)}$, and our estimate of the mode by $\modestate^{(n)}$. Observe that $\mapstate^{(n)} \equiv \mapstate^{(n)}(\Data_{1:n})$ is a deterministic function of the data, and we assume this holds for $\modestate^{(n)} \equiv \modestate^{(n)}(\Data_{1:n})$ also. In general $\modestate^{(n)}$ may depend on additional randomness, say $W_{1:n}$, if for instance it is the output of a stochastic gradient descent algorithm. In that case, our statements should involve conditioning on $W_{1:n}$ but are otherwise unchanged.

Given $n$ data points we denote the proposal by $\prop^{(n)}$, and model the behaviour of our chain at stationarity by considering $\state^{(n)} \sim \target^{(n)}$ and $\nstate^{(n)} \sim \prop^{(n)}(\state^{(n)}, \cdot)$ sampled independently of all other randomness given $\Data_{1:n}$. The following theorem allows us to show that both the computational cost and the acceptance probability of SMH remain stable as $n \to \infty$. See Section \ref{sec:performance-gains} in the Supplement for a proof.

\begin{theorem}  \label{thm:bound-asymptotics}
    Suppose each $\pot_i$ is $(k+1)$-times continuously differentiable, each $\potgradub_{k+1,i} \in L^{k+2}$, and $\E[\sum_{i=1}^{m^{(n)}} \potgradub_{k+1,i}|\Data_{1:n}] = O_\Ptrue(n)$. Likewise, assume each of $\norm{\state^{(n)} - \mapstate^{(n)}}$, $\norm{\state^{(n)} - \nstate^{(n)}}$, and $\norm{\modestate^{(n)} - \mapstate^{(n)}}$ is in $L^{k+2}$, and each of $\E[\norm{\state^{(n)} - \mapstate^{(n)}}^{k+1}|\Data_{1:n}]$, $\E[\norm{\state^{(n)} - \nstate^{(n)}}^{k+1}|\Data_{1:n}]$, and $\norm{\modestate^{(n)} - \mapstate^{(n)}}^{k+1}$ is $O_\Ptrue(n^{-(k+1)/2})$ as $n \to \infty$. Then $\bound$ defined by \eqref{eq:smh-bound} satisfies
    \[
        \E[\bound(\state^{(n)}, \nstate^{(n)}) | \Data_{1:n}] = O_\Ptrue(n^{(1-k)/2}).
    \]
\end{theorem}

For given $\state^{(n)}$ and $\nstate^{(n)}$, recall that the method described in Section \ref{sec:subsampling-via-ppp} requires the computation of at most $N^{(n)} \sim \Poisson(\bound(\state^{(n)}, \nstate^{(n)}))$ factors. Under the conditions of Theorem \ref{thm:bound-asymptotics}, we therefore have
\begin{eqnarray*}
    \E[N^{(n)}|\Data_{1:n}] &=& \E[\E[N^{(n)}|\state^{(n)}, \nstate^{(n)}, \Data_{1:n}]|\Data_{1:n}] \\
    &=& \E[\bound(\state^{(n)}, \nstate^{(n)})|\Data_{1:n}] \\
    &=& O_\Ptrue(n^{(1-k)/2}).
\end{eqnarray*}
In other words, with arbitrarily high probability with respect to the data-generating distribution, SMH requires processing on average only $O(1)$ data points per step for a first-order approximation, and $O(1/\sqrt{n})$ for a second-order one.

This result also ensures that the acceptance probability for SMH does not vanish as $n \to \infty$. Denoting by $\approxtarget^{(n)}_k$ our approximation in the case of $n$ data points, observe that
\begin{multline*}
    0 \leq \E[-\log \afmh(\state^{(n)}, \nstate^{(n)})|\Data_{1:n}] \\
    \leq \E[-\log (1 \wedge \frac{\approxtarget^{(n)}_k(\nstate^{(n)}) \prop^{(n)}(\nstate^{(n)}, \state^{(n)})}{\approxtarget^{(n)}_k(\state^{(n)}) \prop^{(n)}(\state^{(n)}, \nstate^{(n)})})|\Data_{1:n}] \\
    + \E[\bound(\state^{(n)}, \nstate^{(n)})|\Data_{1:n}]
\end{multline*}
Here the second right-hand side term is $O_\Ptrue(n^{(1-k)/2})$ by Theorem \ref{thm:bound-asymptotics}. For a $\approxtarget_k$-reversible proposal the first term is simply $0$, while for a symmetric proposal Theorem \ref{SUPP:thm:solitary-term-asymptotics} in the Supplement shows it is $O_\Ptrue(1)$. In either case, we see that the acceptance probability is stable in the limit of large $n$. In the case of a $\approxtarget_2$-reversible proposal, we in fact have $\E[\afmh(\state^{(n)}, \nstate^{(n)})|\Data_{1:n}] \Pto{\Ptrue} 1$.

Note that both these implications also apply if we use a truncated version of SMH as per Section \ref{sec:geom-erg}. This holds since in general TFMH ensures both that the expected number of factor evaluations is not greater than for FMH, and that the acceptance probability is not less than for FMH.

The conditions of Theorem \ref{thm:bound-asymptotics} hold in realistic scenarios. The integrability assumptions are mild and mainly technical. We will see in Section \ref{sec:experiments} that in practice $\potgradub_{k+1,i} \equiv \potgradub_{k+1}(\Data_i)$ is usually a function of $\Data_i$, in which case
\[
    \E[\sum_{i=1}^{m^{(n)}} \potgradub_{k,i}|\Data_{1:n}] = \sum_{i=1}^n \potgradub_{k+1}(\Data_i) = O_\Ptrue(n)
\]
by the law of large numbers. In general, we might also have one $\potgradub_{k,i}$ for the prior also, but the addition of this term still gives the same asymptotic behaviour.

The condition $\E[\norm{\state^{(n)} - \mapstate^{(n)}}^{k+1}|\Data_{1:n}] = O_\Ptrue(n^{-(k+1)/2})$ essentially states that the posterior must concentrate at rate $O(1/\sqrt{n})$ around the mode. This is a consequence of standard, widely-applicable assumptions that are used to prove Bernstein-von Mises. See Section  \ref{sec:concentration-around-mode} of the Supplement for more details. Note in particular that we do not require our model to be well-specified (i.e. we do not need $\Ptrue = p(\data|\truestate)$ for some $\truestate \in \statespace$). The remaining two $O_\Ptrue$ conditions correspond to the heuristic conditions given in Section \ref{sec:smh}. In particular, the proposal should scale like $1/\sqrt{n}$. We show Section \ref{sec:prop-scale} of the Supplement that this condition holds for the proposals described in Section \ref{sec:prop-choice}. Likewise, $\modestate$ should be distance $O(1/\sqrt{n})$ from the mode. When the posterior is log-concave it can be shown this holds for instance for stochastic gradient descent after performing a single pass through the data \citep[Section 3.4]{baker2017control}. In practice, we interpret this condition to mean that $\modestate$ should be as close as possible to $\mapstate$, but that some small margin for error is acceptable.
\section{Experimental Results} \label{sec:experiments} 

%\todo[inline]{SGLD hard to tune...}
In this section we apply SMH to Bayesian logistic regression. A full description of the model and upper bounds \eqref{eq:grad-upper-bound} we used is given in Section \ref{sec:robust-regression-bounds} of the Supplement. We also provide there an additional application our method to robust linear regression. We chose these models due to the availability of lower bounds on the likelihoods required by Firefly.

% In both cases we write our covariates as $x_i$ and responses as $\data_i$, and our target is the posterior
% \[
%     \target(\state) = p(\state|x_{1:n}, \data_{1:n})
%         \propto p(\state) \prod_{i=1}^n p(\data_i|\state, x_i).
% \]
% For logistic regression we have $x_i \in \R^d$, $\data_i \in \set{0, 1}$, and
% \[
%   p(y_i | \state, x_i) = \Bernoulli(y_i | \frac{1}{1 + \exp(-\state^\top x_i)}).
% \]
% For simplicity we assume a flat prior $p(\state) \equiv 1$, which allows factorising $\target$ like \eqref{eq:base-factorisation} with $m = n$ and $\basetarget_i(\state) = p(\data_i|\state, x_i)$. It is then easy to show that
% \[
%   \pot_i(\state) = -\log \widetilde{\target}_i(\state) = \log(1 + \exp(\state^T x_i)) - \data_i \state^\top x_i.
% \]
% We require upper bounds $\potgradub_{k+1,i}$ of the form \eqref{eq:grad-upper-bound} for these terms. A straightforward calculation shows that
% \[
%     \potgradub_{2,i} = \frac{1}{4} \max_{1 \leq j \leq d} \abs{x_{ij}}^2
%     \quad\quad
%     \potgradub_{3,i} = \frac{1}{6\sqrt{3}} \max_{1 \leq j \leq d} \abs{x_{ij}}^3
% \]
% will do. See Section \ref{sec:logistic-regression-bounds} in the Supplement for the details.

In our experiments we took $d = 10$. For both \smh{1} and \smh{2} we used truncation as described in Section \ref{sec:geom-erg}, with $\transrad=n$. Our estimate of the mode $\modestate$ was computed using stochastic gradient descent. We compare our algorithms to standard MH, Firefly, and Zig-Zag \cite{bierkens2016zigzag}, which all have the exact posterior as the invariant distribution. We used the MAP-tuned variant of Firefly (which also makes use of $\modestate$) with implicit sampling (this uses an algorithmic parameter $q_{d\rightarrow b} = 10^{-3}$; the optimal choice of $q_{d\rightarrow b}$ is an open question) and the lower bounds specified in Section 3.1 of \citet{maclaurin2014firefly}.

Figure \ref{fig:Averagenumberlikelihoodevaluations} (in Section \ref{sec:introduction}) shows the average number of likelihood evaluations per step and confirms the predictions of Theorem \ref{thm:bound-asymptotics}. Figure \ref{fig:nonreversible_iact} displays the effective sample sizes (ESS) for the posterior mean estimate of one regression coefficient, rescaled by execution time. For large $n$, \smh{2} significantly outperforms competing techniques. For all methods except Zig-Zag we used the proposal \eqref{proposalpreconditionedRW} with $\sigma = 1$, which automatically scales according to the concentration of the target.

We also separately considered the performance of the pCN proposal. Figure \ref{fig:reversible_iact} shows the effect of varying $\rho$. As the target concentrates, the Gaussian approximation of the target improves and an independent proposal ($\rho=0$) becomes optimal. Finally, we also illustrate the average acceptance rate when varying $\rho$ in Figure \ref{fig:acceptance_rate}.

Since \smh{2} makes use of a Gaussian approximation to the posterior $\approxtarget_2$, we finally consider the benefit that our method yields over simply using $\approxtarget_2$ directly. Observe in Figure \ref{fig:acceptance_rate} that the acceptance probability of MH with the independent proposal differs non-negligibly from $1$ for reasonably large values of $n$, which indicates that our method yields a non-trivial increase in accuracy. For very large $n$, the discrepancy vanishes as expected and SMH and other subsampling methods based on control variates become less useful in practice. See Section \ref{Supp:Applications} of the Supplement for further results along these lines. We believe however that our approach could form the basis of subsampling methods in more general and interesting settings such as random effect models and leave this as an important piece of future work.

Code to reproduce our experiments is available at \url{github.com/pjcv/smh}.

\begin{figure}[ht]
\vskip 0.2in
\begin{center}
\centerline{\includegraphics[width=.85\columnwidth]{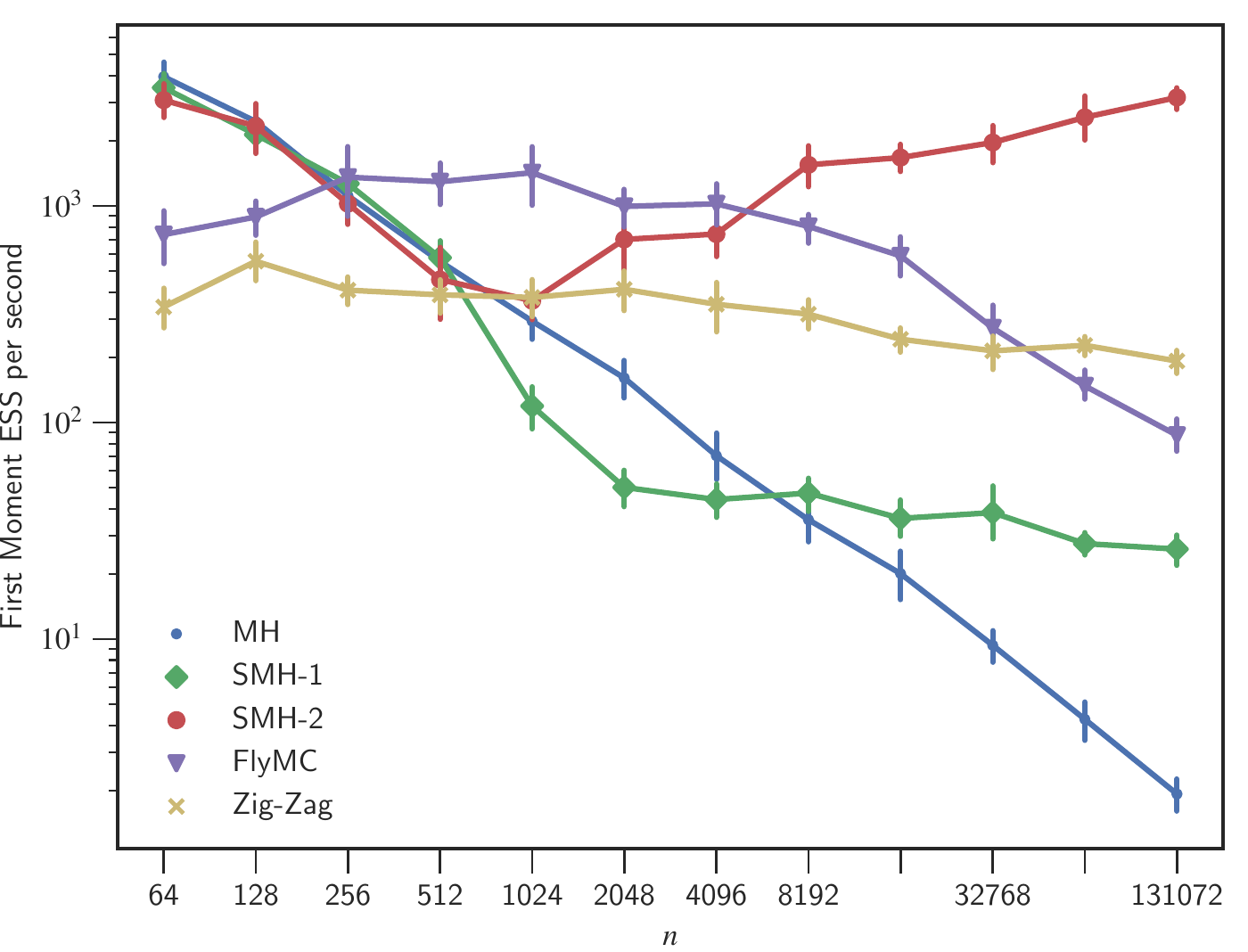}}
\caption{ESS for first regression coefficient, scaled by execution time (higher is better).}
\label{fig:nonreversible_iact}
\end{center}
\vskip -0.2in
\end{figure}

\begin{figure}[ht!]
\vskip 0.2in
\begin{center}
\centerline{\includegraphics[width=.85\columnwidth]{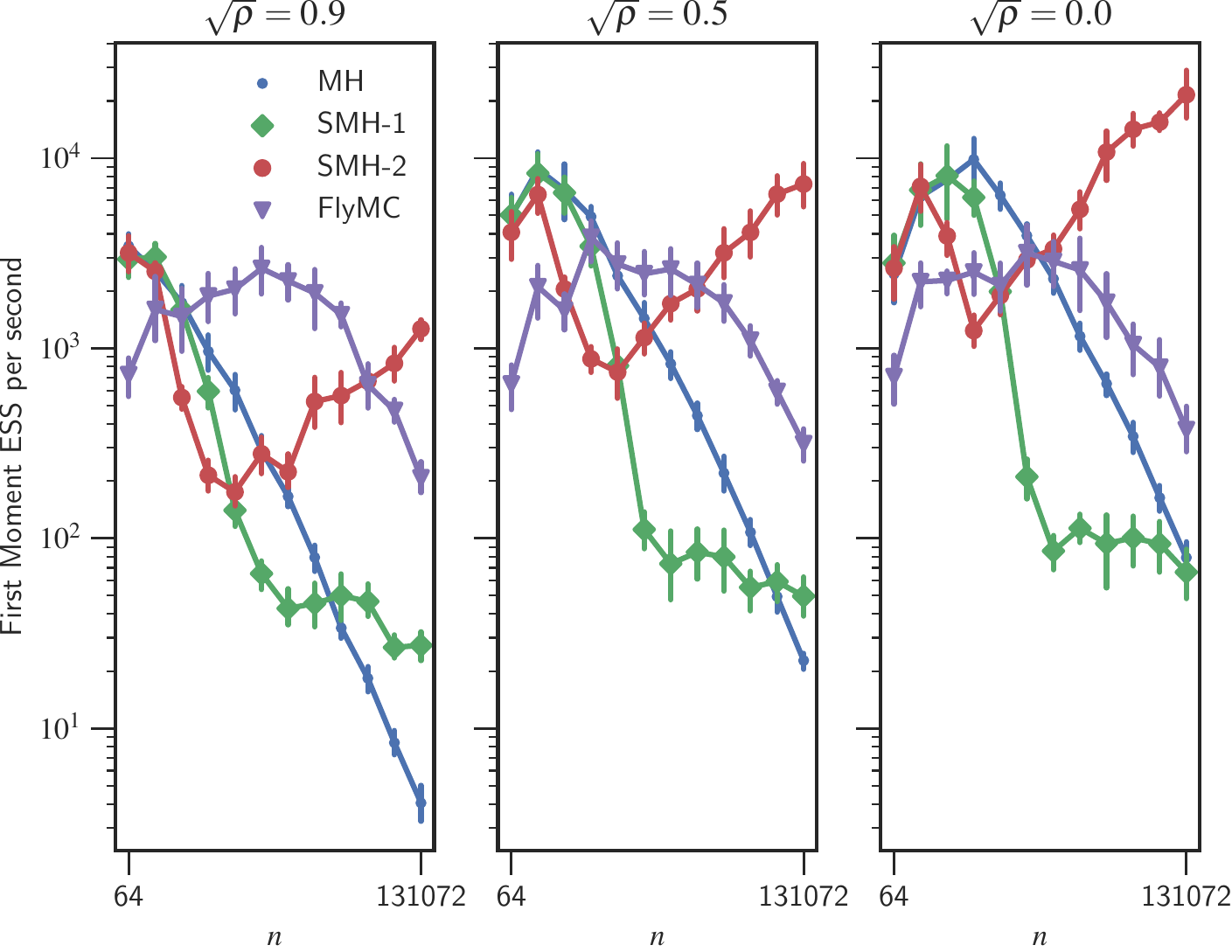}}
\caption{Effect of $\rho$ on ESS. ESS for first regression coefficient, scaled by execution time (higher is better).}
\label{fig:reversible_iact}
\end{center}
\vskip -0.2in
\end{figure}

\begin{figure}[ht!]
\vskip 0.2in
\begin{center}
\centerline{\includegraphics[width=.85\columnwidth]{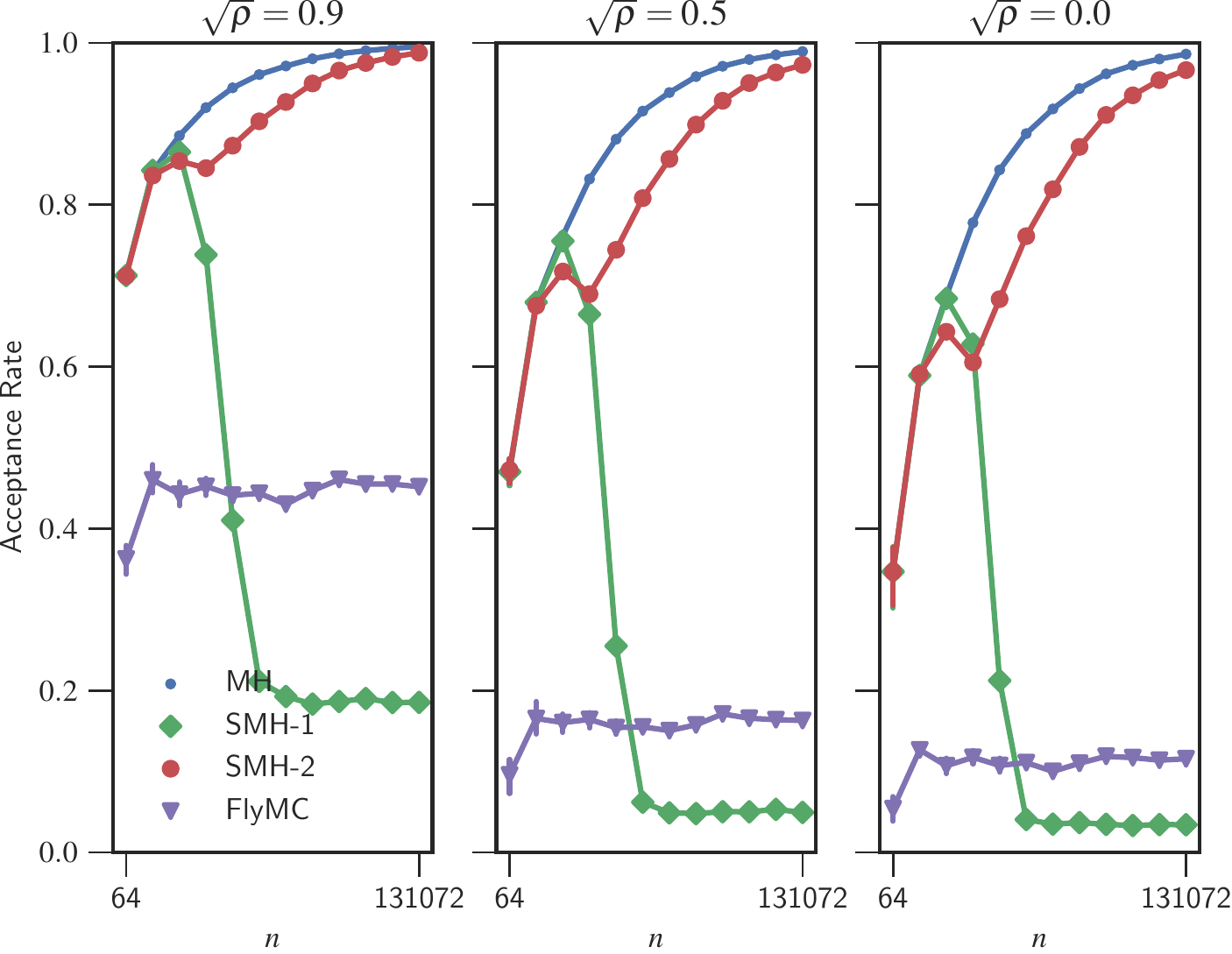}}
\caption{Acceptance rates for pCN proposals.}
\label{fig:acceptance_rate}
\end{center}
\vskip -0.2in
\end{figure}

\newpage

%\section*{Acknowledgements}

\bibliography{main}
\bibliographystyle{icml2019}

\newpage

\onecolumn
\icmltitle{\mytitle: Supplementary Material}
\appendix
\numberwithin{equation}{section}

\section{Guide to Notation} \label{sec:notation}

\begin{table}[h]
  \begin{tabularx}{\linewidth}{l X}
    $a \wedge b$ & $\min\set{a, b}$ \\
    $a \vee b$ & $\max\set{a, b}$ \\
    $B(x, K)$ & Euclidean ball centered at $x$ of radius $K$ \\
    $\ind_A$ & Indicator function of the set $A$ \\
    $\partial_j F(x)$ & $j$-th partial derivative of $F$ at $x$, i.e. $\partial F(x) / \partial x_j$ \\
    $\nabla F(x)$ & Gradient of $F$ at $x$ \\
    $\nabla^2 F(x)$ & Hessian of $F$ at $x$ \\
    $\norm{\cdot}$ & The $\ell^2$ norm \\
    $\norm{\cdot}_1$ & The $\ell^1$ norm \\
    $\norm{\cdot}_\infty$ & The supremum norm \\
    $\opnorm{\cdot}$ & The operator norm with respect to $\norm{\cdot}$ on the domain and range \\
    $\idmat$ & Identity matrix \\
    $A \prec B$ & $B - A$ is symmetric positive-definite \\
    $A \preceq B$ & $B - A$ is symmetric nonnegative-definite \\
    $a(x) \asymp b(x)$ as $x \to x_0$ & $\lim_{x \to x_0} a(x)/b(x) = 1$ \\
    $a(x) = O(b(x))$ as $x \to x_0$ & $\limsup_{x \to x_0} \abs{a(x)/b(x)} < \infty$ \\
    $a(x) = \Theta(b(x))$ as $x \to x_0$ & $a(x) = O(b(x))$ as $x \to x_0$ and $\liminf_{x \to x_0} \abs{a(x) / b(x)} > 0$. (Note that similar notation is used for our state space $\statespace$, but the meaning will always be clear from context.)  \\
    $x \ll y$ & (Informal) $x$ is much smaller than $y$ \\
    $x \approx y$ & (Informal) $x$ is approximately equal to $y$ \\
    $\Leb$ & The Lebesgue measure \\
    a.s. & Almost surely \\
    i.i.d. & Independent and identically distributed \\
    $X_n \Pto{\P} X$ & $X_n$ converges to $X$ in $\P$-probability \\
    $X_n = O_\P(a_n)$ & $X_n / a_n$ is $\P$-tight, i.e. for all $\epsilon > 0$ there exists $c>0$ such that $\P(\abs{X_n/a_n} < c) > 1 - \epsilon$ for all $n$  \\
    $X_n = o_\P(a_n)$ & $X_n / a_n \Pto{\P} 0$ \\
    $\E[X]$ & Expectation of a random variable $X$ \\
    $\E[X;A]$ & $\E[X \ind_A]$ \\
    $L^p$ & The space of random variables $X$ such that $\E[\abs{X}^p] < \infty$ \\
    $L^p(\mu)$ & The space of real-valued test functions $f$ such that $f(X) \in L^p$ where $X \sim \mu$ \\
  \end{tabularx}
\end{table}

We also use multi-index notation to express higher-order derivatives succinctly. Specifically, for $\midx = (\midx_1, \cdots, \midx_d) \in \Z_{\geq 0}^d$ and $\state = (\state_1, \cdots, \state_d) \in \statespace$, we define
\[
  \abs{\midx} := \sum_{i=1}^d \midx_i \quad\quad
  \midx! := \prod_{i=1}^d \midx_i! \quad\quad
  \state^\midx := \prod_{i=1}^d \state_i^{\midx_i} \quad\quad
  \partial^\midx := \frac{\partial^{\abs{\midx}}}{\partial^{\midx_1} \cdots
  \partial^{\midx_d}}.
\]

\section{Factorised Metropolis--Hastings} \label{sec:acc-probs}

Note that the definition \eqref{eq:fmh_acc_prob} of $\afmh(\state, \nstate)$ technically does not apply when
$\target(\state) \prop(\state, \nstate) = 0$. For concreteness, like
\citet{hastings1970monte}, we therefore define explicitly
\[
  \afmh(\state, \nstate) := \begin{cases}
    \prod_{i=1}^m 1 \wedge \frac{\target_i(\nstate) \prop_i(\nstate, \state)}{\target_i(\state) \prop_i(\state, \nstate)} & \text{if each $\target_i(\state) \prop_i(\state, \nstate) \neq 0$} \\
    1 & \text{otherwise},
  \end{cases}
\]
and take $\amh(\state, \nstate)$ to be the case when $m = 1$. We still take
$\atfmh(\state, \nstate)$ to be defined by \eqref{eq:tfmh-acc-prob}. We first establish a useful preliminary Proposition.

\begin{proposition} \label{prop:afmh-amh-relation}
  For all $\state, \nstate \in \statespace$, $\afmh(\state, \nstate) = \amh(\state,
  \nstate)(\afmh(\state, \nstate) \vee \afmh(\nstate, \state))$.
\end{proposition}

\begin{proof}
  The cases where $\target_i(\state) \prop_i(\state, \nstate) = 0$ or $\target_i(\nstate)
  \prop_i(\nstate, \state) = 0$ for some $i$ are immediate from the definition above.
  Otherwise, since $(1 \wedge c)^{-1} = 1 \vee c^{-1}$ for all $c > 0$,
  \begin{eqnarray*}
    \amh(\state, \nstate)^{-1}
    &=& \left(1 \wedge \prod_{i=1}^m \frac{\target_i(\nstate) \prop_i(\nstate, \state)}{\target_i(\state) \prop_i(\state, \nstate)}\right)^{-1} \\
    &=& 1 \vee \prod_{i=1}^m \frac{\target_i(\state) \prop_i(\state, \nstate)}{\target_i(\nstate) \prop_i(\nstate, \state)},
  \end{eqnarray*}
  and hence
  \begin{eqnarray*}
    \frac{\afmh(\state, \nstate)}{\amh(\state, \nstate)}
    &=& \afmh(\state, \nstate) \vee \left(\afmh(\state, \nstate) \prod_{i=1}^m \frac{\target_i(\state) \prop_i(\state, \nstate)}{\target_i(\nstate) \prop_i(\nstate, \state)}\right) \\
    &=& \afmh(\state, \nstate) \vee \left(\prod_{i=1}^m (1 \wedge \frac{\target_i(\nstate) \prop_i(\nstate, \state)}{\target_i(\state) \prop_i(\state, \nstate)}) \prod_{i=1}^m \frac{\target_i(\state) \prop_i(\state, \nstate)}{\target_i(\nstate) \prop_i(\nstate, \state)}\right) \\
    &=& \afmh(\state, \nstate) \vee \left(\prod_{i=1}^m 1 \wedge \frac{\target_i(\state) \prop_i(\state, \nstate)}{\target_i(\nstate) \prop_i(\nstate, \state)}\right) \\
    &=& \afmh(\state, \nstate) \vee \afmh(\nstate, \state)
  \end{eqnarray*}
  which gives the result.
\end{proof}

\begin{corollary}
    For all $\state, \nstate \in \statespace$, $\afmh(\state, \nstate) \leq \amh(\state, \nstate)$.
\end{corollary}

\subsection{Reversibility} \label{SUPP:fmh-reversible}

To show reversibility for $\Pfmh$ and $\Ptfmh$, we will use the standard result (see e.g.
\citep[Lemma 3.4]{geyer1998markov}) that a kernel of the form
\[
  \Pgeneric(\state, A)
  = \left(1 - \int \prop(\state, \nstate) \ageneric(\state, \nstate) d\nstate\right) \ind_A(\state)
    + \int_A \prop(\state, \nstate) \ageneric(\state, \nstate) d\nstate
\]
is reversible if $\target(\state) \prop(\state, \nstate) \ageneric(\state, \nstate)$ is
symmetric in $\state$ and $\nstate$. It is straightforward to show for instance that
\begin{equation} \label{eq:mh-symm}
  \target(\state) \prop(\state, \nstate) \amh(\state, \nstate)
    = \target(\nstate) \prop(\nstate, \state) \amh(\nstate, \state),
\end{equation}
which is immediate if either $\target(\state) = 0$ or $\target(\nstate) = 0$, and
otherwise
\begin{eqnarray*}
  \target(\state) \prop(\state, \nstate) \amh(\state, \nstate)
    &=& \target(\state) \prop(\state, \nstate) \left(1 \wedge \frac{\target(\nstate) \prop(\nstate, \state)}{\target(\state) \prop(\state, \nstate)}\right) \\
    &=& \target(\state) \prop(\state, \nstate) \wedge \target(\nstate) \prop(\nstate, \state).
\end{eqnarray*}

We use this result to establish reversiblity of $\Pfmh$. This result is standard but we include it here for completeness.
\begin{proposition} \label{prop:pfmh-reversible}
  $\Pfmh$ is $\target$-reversible.
\end{proposition}

\begin{proof}
  By Proposition \ref{prop:afmh-amh-relation}
  \[
    \target(\state) \prop(\state, \nstate) \afmh(\state, \nstate)
      = \target(\state) \prop(\state, \nstate) \amh(\state, \nstate)(\afmh(\state,
        \nstate) \vee \afmh(\nstate, \state)),
  \]
  which is symmetric in $\state$ and $\nstate$ by \eqref{eq:mh-symm}.
\end{proof}

\begin{proposition} \label{prop:ptfmh-reversible}
  If $\bound(\state, \nstate)$ is symmetric in $\state$ and $\nstate$, then $\Ptfmh$ is $\target$-reversible.
\end{proposition}

\begin{proof}
    Simply write
    \[
        \atfmh(\state, \nstate) = \ind(\bound(\state, \nstate) < R) \afmh(\state, \nstate) + \ind(\bound(\state, \nstate) \geq R) \amh(\state, \nstate).
    \]
    The result then follows from the symmetry of the indicator functions, \eqref{eq:mh-symm}, and the proof of Proposition \ref{prop:pfmh-reversible}.
\end{proof}

\subsection{Ergodic Properties} \label{SUPP:sec:ergodic-props}

We provide a brief background to the theory of $\varphi$-irreducible Markov Chains. See \cite{meyn2009markov} for a comprehensive treatment.

For a transition kernel $\Pgeneric$, we inductively define the transition kernel $\Pgeneric^k$ for $k \geq 1$ by setting $\Pgeneric^1 := \Pgeneric$ and
\[
  \Pgeneric^k(\state, A) := \int \Pgeneric(\state, d\nstate) \Pgeneric^{k-1}(\nstate, A) d\nstate
\]
for $k > 1$, where $\state \in \statespace$ and $A \subseteq \statespace$ is measurable.
Given a nontrivial measure $\varphi$ on $\statespace$, we say $\Pgeneric$ is
\emph{$\varphi$-irreducible} if $\varphi(A) > 0$ implies $\Pgeneric^k(\state, A) > 0$ for
some $k \geq 1$.  For $\varphi$-irreducible $\Pgeneric$, we define a \emph{$k$-cycle} of
$\Pgeneric$ to be a partition $D_1, \cdots, D_k, N$ of $\statespace$ such that $\varphi(N)
= 0$, and for all $1 \leq i \leq k$, if $\state \in D_i$ then $\Pgeneric(\state, D_{i+1})
= 1$. (Here $i + 1$ is meant modulo $k$.) If there exists a $k$-cycle with $k > 1$, we say
that $\Pgeneric$ is \emph{periodic}; otherwise it is \emph{aperiodic}.

If $\Pgeneric$ is $\varphi$-irreducible and aperiodic and has invariant distribution $\target$, we say $\Pgeneric$ is \emph{geometrically ergodic} if there exists constants $\rho < 1$, $C < \infty$, and a $\target$-a.s. finite function $V \geq 1$ such that
\[
  \norm{\Pgeneric^k(\state, \cdot) - \target}_V \leq C \, V(\state) \rho^k
\]
for all $\state \in \statespace$ and $k \geq 1$. Here $\norm{\cdot}_V$ denotes the $V$-norm on signed measures defined by
\[
    \norm{\mu}_V = \sup_{\abs{f} \leq V} \abs{\target(f)},
\]
where $\target(f) := \int f(\state) \target(d\state)$. By \citep[Proposition 2.1]{roberts1997geometric}, this is equivalent to the apparently weaker condition that there exist some constant $\rho > 0$ and $\target$-a.s. finite function $M$ such that
\[
  \TV{\Pgeneric^k(\state, \cdot) - \target} \leq M(\state) \rho^k
\]
for all $\state \in \statespace$ and $k \geq 1$, where $\TV{\cdot}$ denotes the total variation distance on signed measures.

Our interest in geometric ergodicity is largely due to the implications it has for the \emph{asymptotic variance} of the ergodic averages produced by a transition kernel. Suppose $(\state_k)_{k\geq 1}$ is a stationary Markov chain with transition kernel $\Pgeneric$ having invariant distribution $\target$. For $f \in L^2(\target)$, the asymptotic variance for the ergodic averages of $f$ is defined 
\[
    \var(f, \Pgeneric) := \lim_{k \to \infty} \Var\left(\sqrt{k} (\frac{1}{k} \sum_{i=1}^k f(\state_k) - \target(f))\right) = \lim_{k \to \infty} \frac{1}{k} \Var(\sum_{i=1}^k f(\state_k)).
\]
We abuse notation a little and denote the variance of $f(\state)$ where $\state \sim \target$ by $\var(f, \target)$.

Of interest is also the (right) \emph{spectral gap}, which for a $\target$-reversible transition kernel $\Pgeneric$ is defined
\[
    \Gap(\Pgeneric) := \inf_{f \in L^2(\target) : \target(f) = 0} \frac{\int \int \frac{1}{2}(f(\state) - f(\nstate))^2  \target(d\state) \Pgeneric(\state, d\nstate)}{\int f(\state)^2 \target(d\state)}.
\]

Finally, it is convenient to define the MH rejection probability
\[
  \rmh(\state) := 1 - \int \prop(\state, \nstate) \amh(\state, \nstate) d\nstate,
\]
and similarly $\rfmh$ and $\rtfmh$ for FMH and TFMH.

\begin{proposition} \label{prop:fmh-irr-aper-when-mh-is}
  $\Ptfmh$ is $\varphi$-irreducible and aperiodic whenever $\Pmh$ is.
\end{proposition}

\begin{proof}
    We use throughout the easily verified facts $\afmh(\state, \nstate) \leq \atfmh(\state, \nstate) \leq \amh(\state, \nstate)$ and $\rfmh(\state) \geq \rtfmh(\state) \geq \rmh(\state)$ for all $\state, \nstate \in \statespace$. See Proposition \ref{prop:afmh-amh-relation}.

    For $\varphi$-irreducibility, first note that if $\amh(\state, \nstate) > 0$ then $\atfmh(\state, \nstate) > 0$. This holds since if $\atfmh(\state, \nstate) = 0$, then either $\amh(\state, \nstate) = 0$ or $\afmh(\state, \nstate) = 0$. In the latter case we must have some $\target_i(\nstate) \prop_i(\nstate, \state) = 0$, so that $\target(\nstate) \prop(\state, \nstate) = 0$, and hence again $\amh(\state, \nstate) = 0$.

  We now show by induction on $k \in \Z_{\geq 1}$ that for all $\state \in \statespace$,
  $\Pmh^k(\state, A) > 0$ implies $\Ptfmh^k(\state, A) > 0$. For $k = 1$, suppose
  $\Pmh(\state, A) > 0$. Then either $\rmh(\state) \ind_A(\state) > 0$ or $\int_A
  \prop(\state, \nstate) \amh(\state, \nstate) d\nstate > 0$. In the former case we
  we have
  \[
    \rtfmh(\state) \ind_A(\state) \geq \rmh(\state) \ind_A(\state) > 0.
  \]
  In the latter case the above considerations give
  \begin{eqnarray*}
    \Leb(\set{\nstate \in A \mid \prop(\state, \nstate) \atfmh(\state, \nstate) > 0})
      &=& \Leb(\set{\nstate \in A \mid \prop(\state, \nstate) \amh(\state, \nstate) > 0}) \\
      &>& 0.
  \end{eqnarray*}
  Either way we have $\Ptfmh(\state, A) > 0$.

  Suppose now $\Pmh^{k-1}(\state, A) > 0$ implies $\Ptfmh^{k-1}(\state, A) > 0$. Then observe
  \[
    \Pmh^k(\state, A) = 
    \rmh(\state) \Pmh^{k-1}(\state, A)  
    + \int \prop(\state, \nstate) \amh(\state, \nstate) \Pmh^{k-1}(\nstate, A) d\nstate
  \]
  and likewise \emph{mutatis mutandis} for $\Ptfmh^k(\state, A)$. Thus if $\Pmh^k(\state,
  A) > 0$, one possibility is $\rmh(\state) \Pmh^{k-1}(\state, A) > 0$, which implies
  $\rtfmh(\state) > 0$ and, by the induction hypothesis, $\Ptfmh^{k-1}(\state, A) > 0$. The
  only other possibility is
  \begin{eqnarray*}
    \Leb(\set{\nstate \in \statespace \mid \prop(\state, \nstate) \atfmh(\state, \nstate) \Ptfmh^{k-1}(\nstate, A) > 0})
      &=& \Leb(\set{\nstate \in \statespace \mid \prop(\state, \nstate) \amh(\state, \nstate) \Pmh^{k-1}(\nstate, A) > 0}) \\
      &>& 0,
  \end{eqnarray*}
  again by the induction hypothesis. Either way, as desired $\Ptfmh^k(\state, A) > 0$. It
  now follows that $\Ptfmh$ is $\varphi$-irreducible when $\Pmh$ is.

  Now suppose $\Pmh$ and hence $\Ptfmh$ is $\varphi$-irreducible. If $\Ptfmh$ is periodic,
  then there exists a $k$-cycle $D_1, \cdots, D_k, N$ for $\Ptfmh$ with $k > 1$. But now if
  $\state \in D_i$, then $\ind_{D_{i+1}}(\state) = 0$ and so
  \begin{eqnarray*}
    \Pmh(\state, D_{i+1})
      &=& \int_{D_{i+1}} \prop(\state, \nstate) \amh(\state, \nstate) d\nstate \\
      &\geq& \int_{D_{i+1}} \prop(\state, \nstate) \atfmh(\state, \nstate) d\nstate \\
      &=& \Ptfmh(\state, D_{i+1}) \\
      &=& 1.
  \end{eqnarray*}
  Thus the same partition is a $k$-cycle for $\Pmh$ which is therefore periodic.
\end{proof}

\begin{theorem}
  If $\Pmh$ is $\varphi$-irreducible, aperiodic, and geometrically ergodic, then $\Ptfmh$ is too if
  \[
    \delta := \inf_{\bound(\state, \nstate) < \transrad} \afmh(\state, \nstate) \vee \afmh(\nstate, \state) > 0.
  \]
  In this case, $\mathrm{Gap}(\Pfmh) \geq \delta \mathrm{Gap}(\Pmh)$, and for $f \in L^2(\target)$
  \[
    \var(f, \Ptfmh) \leq (\delta^{-1} - 1) \var(f, \target) + \delta^{-1} \var(f, \Pmh).
  \]
\end{theorem}

\begin{proof}
Our proof of this result is similar to \citep[Proposition 1]{banterle2015accelerating}, but differs in its use of Proposition \ref{prop:afmh-amh-relation} to express the relationship between MH and FMH exactly.

  For $\state \in \statespace$, let
  \[
    \mathcal{R}(\state) := \set{\nstate \in \statespace \mid \bound(\state, \nstate) < \transrad}.
  \]
  Whenever $\state \in \statespace$ and $A \subseteq \statespace$ is measurable,
  \begin{align*}
    \Ptfmh(\state, A) &=
        \begin{aligned}[t]
            &\rtfmh(\state) \ind_A(\state) + \int_{\mathcal{R}(\state) \cap A} \prop(\state, \nstate) \amh(\state, \nstate) (\afmh(\state, \nstate) \vee \afmh(\nstate, \state)) d\nstate \\
            &+ \int_{\mathcal{R}(\state)^c \cap A} \prop(\state, \nstate) \amh(\state, \nstate) d\nstate
        \end{aligned} \\
    &\geq \rmh(\state) \ind_A(\state) + \delta \int_{\mathcal{R}(\state) \cap A} \prop(\state, \nstate) \amh(\state, \nstate) d\nstate + \int_{\mathcal{R}(\state)^c \cap A} \prop(\state, \nstate) \amh(\state, \nstate) d\nstate \\
      &\geq \delta \Pmh(\state, A).
  \end{align*}
  The last line follows since certainly $\delta \leq 1$.
  
  Suppose $\delta > 0$. If $\Pmh$ is geometrically ergodic, then \citep[Theorem 1]{jones2014convergence} entails that $\Ptfmh$ is geometrically ergodic also. The remaining claims follow directly from \citep[Lemma 32]{andrieu2013uniform}.
\end{proof}

\section{Fast Simulation of Bernoulli Random Variables}
For sake of completeness, we provide here the proof of validity of Algorithm \ref{alg:poisson-sampling}. It combines the Fukui-Todo procedure \cite{fukui2009order} with a thinning argument.
\begin{proposition} \label{prop:poisson-subsampling}
  If 
\begin{itemize}[label={}, leftmargin=*]
  \item $N \sim \Poisson\left(\bound(\state, \nstate)\right)$
  \item $X_1, \cdots, X_N \iid \Categorical((\bound_i(\state,
    \nstate) / \bound(\state, \nstate))_{1\leq i \leq m})$
  \item $B_j \sim \Bernoulli(\intensity_{X_j}(\state, \nstate) /
   \bound_{X_j}(\state, \nstate))$ independently for $1 \leq j \leq N$
\end{itemize}
  then $\P(B = 0) = \afmh(\state, \nstate)$ where $B = \sum_{j=1}^N B_j$ (and $B = 0$ if $N = 0$).
\end{proposition}

\begin{proof}
  Letting
  \[
    \lambda(\state, \nstate) := \sum_{i=1}^m \lambda_i(\state, \nstate),
  \]
  our goal is to show that $\P(B = 0) = \exp(-\lambda(\state, \nstate))$.
  For brevity we omit all dependences on $\state$ and $\nstate$ in the following.

  Observe the random variables $B_j$'s are i.i.d. with
  \[
    \P(B_j = 0)
      = \sum_{i=1}^m \underbrace{\P(X_j = i)}_{= \overline{\lambda}_i/\overline{\lambda}}
        \underbrace{\P(B_j = 0|X_j = i)}_{= 1 - \lambda_i/\overline{\lambda}_i}
      = \frac{\overline{\lambda} - \lambda}{\overline{\lambda}}.
  \]
  Thus
  \begin{eqnarray*}
    \P(B = 0)
      &=& \sum_{\ell=0}^\infty \underbrace{\P(N=\ell)}_{= \exp(-\overline{\lambda})\overline{\lambda}^\ell/\ell!} \P(B_1=0)^\ell \\
      &=& \exp(-\overline{\lambda}) \sum_{\ell=0}^\infty \frac{(\overline{\lambda} - \lambda)^\ell}{\ell!} \\
      &=& \exp(-\lambda)
  \end{eqnarray*}
  as desired.
\end{proof}

\section{Upper Bounds} \label{SUPP:sec:upper-bounds}

We refer the reader to Section \ref{sec:notation} for an explanation of multi-index notation $\midx$.

\begin{proposition} \label{SUPP:prop:upper-bounds}
  If each $\pot_i$ is $(k+1)$-times continuously differentiable with
    \[
      \potgradub_{k+1, i}
        \geq \sup_{\substack{\state \in \statespace \\ \abs{\midx} = k+1}}
          \abs{\partial^\midx \pot_i(\state)},
    \]
    then
    \[
        -\log \left(1 \wedge \frac{\target_i(\nstate) \approxtarget_{k,i}(\state)}{\target_i(\state) \approxtarget_{k,i}(\nstate)}\right) \leq (\norm{\state - \modestate}_1^{k+1} + \norm{\nstate - \modestate}_1^{k+1}) \frac{\potgradub_{k+1,i}}{(k+1)!}.
    \]
\end{proposition}

\begin{proof}
  We have
  \begin{eqnarray*}
    -\log \left(1 \wedge \frac{\target_i(\nstate) \approxtarget_{k,i}(\state)}{\target_i(\state) \approxtarget_{k,i}(\nstate)}\right)
      &=& 0 \vee (\pot_{i}(\nstate) - \approxpot_{k,i}(\nstate) - \pot_{i}(\state) + \approxpot_{k,i}(\state)) \\
      &\leq& \abs{\pot_{i}(\nstate) - \approxpot_{k,i}(\nstate)} + \abs{\pot_{i}(\state) - \approxpot_{k,i}(\state)}.
  \end{eqnarray*}
  Notice that $\pot_i(\state) - \approxpot_{k,i}(\state)$ is just the remainder of a Taylor
  expansion. As such, for each $\state$, Taylor's remainder theorem gives for some $\widetilde{\state}\in \statespace$
  \begin{eqnarray*}
    \abs{\pot_i(\state) - \approxpot_{k,i}(\state)}
    &=& \Abs{\frac{1}{(k+1)!} \sum_{\abs{\midx} = k+1} \partial^\midx \pot_i(\widetilde{\state}) (\state - \modestate)^\midx} \\
      &\leq& \frac{\potgradub_{k+1,i}}{(k+1)!} \sum_{\abs{\midx} = k+1} \frac{\abs{(\state - \modestate)^\midx}}{\midx!} \\
      &\leq& \frac{\potgradub_{k+1,i}}{(k+1)!} \norm{\state - \modestate}_1^{k+1}.
  \end{eqnarray*}
The result now follows.
\end{proof}

\section{Reversible Proposals}

\subsection{General Conditions for Reversibility}

We can handle both the first and second-order cases with the following Proposition.

\begin{proposition} \label{SUPP:prop:reversible-prop-conds}
  Suppose
  \[
    \prop(\state,  \nstate) = \Normal(\nstate \mid \rpropmat\state + \rpropvec, \rpropcov)
  \] 
  and
  \[
    -\log \hat{\target}(\state) = \frac{1}{2} \state^\top D \state + e^\top \state + \mathrm{const}
  \]
  where $\rpropmat, \rpropcov, D \in \R^{d \times d}$ with $\rpropcov \succ 0$, and $\rpropvec, e \in \R^d$.  Then $\prop$ is
  $\hat{\target}$-reversible if and only if the following conditions hold:
  \begin{eqnarray}
    \rpropmat^\top \rpropcov^{-1} &=& \rpropcov^{-1} \rpropmat \label{eq:reversible-cond-1} \\
    \rpropmat^2 &=& \idmat - \rpropcov  D \label{eq:reversible-cond-2} \\
    (\rpropmat^\top + \idmat) \rpropvec &=& -\rpropcov e, \label{eq:reversible-cond-3} 
  \end{eqnarray}
  where $\idmat \in \R^{d \times d}$ is the identity matrix.
\end{proposition}

\begin{proof}
  Let
  \[
    F(\state, \nstate) := -\log \hat{\target}(\state) - \log \prop(\state, \nstate).
  \]
  Note that $\prop$ is $\hat{\target}$-reversible precisely when $F$ is symmetric in its arguments.
  Since $F$ is a polynomial of the form
  \begin{equation} \label{eq:reversible-polynomial}
    F(\state, \nstate) = \frac{1}{2} \state^\top J \state + \frac{1}{2} \nstate^\top K \nstate + \state^\top L \nstate + m^\top \state + n^\top \nstate + \mathrm{const},
  \end{equation}
  where $J, K, L \in \R^{d\times d}$ and $m, n \in \statespace$, then by equating coefficients it follows that $F(\state, \nstate) = F(\nstate, \state)$ precisely when
  \begin{eqnarray}
    J &=& K \label{eq:symm-cond-1} \\ 
    L &=& L^\top \label{eq:symm-cond-2}\\
    m &=& n. \label{eq:symm-cond-3}
  \end{eqnarray}
  Now, we can expand
  \begin{eqnarray*}
    -\log \prop(\state, \nstate)
      &=& \frac{1}{2} (\nstate - \rpropmat\state - \rpropvec)^\top \rpropcov^{-1} (\nstate - \rpropmat\state - \rpropvec) + \text{const} \\
      &=& \frac{1}{2} \nstate^\top \rpropcov^{-1}\nstate - (\rpropmat\state + \rpropvec)^\top \rpropcov^{-1} \nstate + \frac{1}{2} (\rpropmat\state + \rpropvec)^\top \rpropcov^{-1} (\rpropmat\state + \rpropvec) + \text{const} \\
      &=& \frac{1}{2} \state^\top \rpropmat^\top \rpropcov^{-1} \rpropmat\state + \frac{1}{2} \nstate^\top \rpropcov^{-1}\nstate
        - \state^\top \rpropmat^\top \rpropcov^{-1} \nstate + \rpropvec^\top \rpropcov^{-1} \rpropmat\state - \rpropvec^\top \rpropcov^{-1} \nstate + \frac{1}{2} \rpropvec^\top \rpropcov^{-1} \rpropvec\\
     && + \text{const}
  \end{eqnarray*}
  Since $-\log \prop(\state, \nstate)$ must be the only source of terms in
  \eqref{eq:reversible-polynomial} containing both $\state$ and $\nstate$, we see immediately that
  \[
    L = -\rpropmat^\top \rpropcov^{-1},
  \]
  and thus from \eqref{eq:symm-cond-2} we have $-\rpropmat^\top \rpropcov^{-1} = -(\rpropcov^{-1})^\top
  \rpropmat$. Since $\rpropcov \succ 0$, $\rpropcov^{-1}$ is symmetric and this condition becomes
  \eqref{eq:reversible-cond-1}. Next we see that
  \begin{eqnarray*}
    J &=& \rpropmat^\top \rpropcov^{-1} \rpropmat + D \\
    K &=& \rpropcov^{-1},
  \end{eqnarray*}
  and from \eqref{eq:symm-cond-1} and \eqref{eq:reversible-cond-1} we require $\rpropcov^{-1}\rpropmat^2 + D = \rpropcov^{-1}$, or equivalently \eqref{eq:reversible-cond-2}. Finally, since
  \begin{eqnarray*}
    m &=& \rpropmat^\top \rpropcov^{-1} \rpropvec + e \\
    n &=& -\rpropcov^{-1}\rpropvec,
  \end{eqnarray*}
  we require from \eqref{eq:symm-cond-3} that $\rpropmat^\top \rpropcov^{-1} \rpropvec + e = -\rpropcov^{-1} \rpropvec$, which combined with \eqref{eq:reversible-cond-1} gives \eqref{eq:reversible-cond-3}.
  
  Since \eqref{eq:symm-cond-1}, \eqref{eq:symm-cond-2}, and \eqref{eq:symm-cond-3} are
  necessary and sufficient for symmetry of $F$, we see that \eqref{eq:reversible-cond-1},
  \eqref{eq:reversible-cond-2}, and \eqref{eq:reversible-cond-3} are necessary and
  sufficient for reversibility also.
\end{proof}

We now specialise this to the first and second-order cases.

\subsection{First-Order Case}

When $k = 1$ we have
\[
  -\log \approxtarget(\state) = \approxpot_{1}(\state) = \pot(\modestate) + \nabla \pot(\modestate)^\top (\state - \modestate),
\]
so that
\begin{eqnarray*}
  D &=& 0 \\
  e &=& \nabla \pot (\modestate),
\end{eqnarray*}
and conditions \eqref{eq:reversible-cond-1}, \eqref{eq:reversible-cond-2}, and \eqref{eq:reversible-cond-3} become
\begin{eqnarray*}
  \rpropmat^\top \rpropcov^{-1} &=& \rpropcov^{-1} \rpropmat \\
  \rpropmat^2 &=& \idmat \\
  (\rpropmat^\top + \idmat) \rpropvec &=& -\rpropcov \nabla \pot(\modestate).
\end{eqnarray*}

\subsection{Second-Order Case}

When $k = 2$,
\[
  -\log \approxtarget(\state) = \approxpot_{2}(\state) = \pot(\modestate) + \nabla \pot(\modestate)^\top (\state - \modestate) + \frac{1}{2} (\state - \modestate)^\top \nabla^2 \pot(\modestate) (\state - \modestate).
\]
In this case
\begin{eqnarray*}
  D &=& \nabla^2 \pot(\modestate) \\
  e &=& \nabla \pot(\modestate) - \nabla^2 \pot(\modestate)^\top \modestate,
\end{eqnarray*}
so conditions \eqref{eq:reversible-cond-1}, \eqref{eq:reversible-cond-2}, and \eqref{eq:reversible-cond-3} become
\begin{eqnarray*}
  \rpropmat^\top \rpropcov^{-1} &=& \rpropcov^{-1} \rpropmat \\
  \rpropmat^2 &=& \idmat - \rpropcov \nabla^2 \pot(\modestate) \\
  (\rpropmat^\top + \idmat) \rpropvec &=& \rpropcov(\nabla^2 \pot(\modestate)^\top \modestate - \nabla \pot(\modestate)).
\end{eqnarray*}
A common setting has $\nabla^2 \pot(\modestate) \succ 0$, $\rpropmat = \rpropmat^\top$, and $\rpropmat + \idmat$ invertible. In this case the latter two conditions become
\begin{eqnarray*}
    \rpropcov &=& (\idmat - \rpropmat^2)[\nabla^2 \pot(\modestate)]^{-1} \\
    \rpropvec &=& (\idmat - \rpropmat) (\modestate - [\nabla^2 \pot(\modestate)]^{-1} \nabla \pot(\modestate)).
\end{eqnarray*}

\subsection{Decreasing Norm Property}
Under usual circumstances for both first and second-order approximations, when
$\norm{\state}$ is large, a $\approxtarget$-reversible $\prop$ will propose $\nstate \sim
\prop(\state, \cdot)$ with smaller norm than $\state$.  This is made precise in the
following Proposition:

\begin{proposition} \label{SUPP:prop:decreasing-norm}
    Suppose
    \[
        \prop(\state, \nstate) = \Normal(\nstate \mid \rpropmat\state + \rpropvec, \rpropcov)
    \] 
    and
    \[
        -\log \hat{\target}(\state) = \frac{1}{2} \state^\top D \state + e^\top \state + \mathrm{const},
    \]
    where $\rpropmat = \rpropmat^\top$ is symmetric, $\rpropcov \succ 0$, and $D \succeq 0$. If $\prop$ is $\hat{\pi}$-reversible, then $\opnorm{\rpropmat} \leq 1$. If $D \succ 0$ is strict, then $\opnorm{\rpropmat} < 1$ is strict too. In this case, if $\nstate \sim \prop(\state, \cdot)$, then $\norm{\state} - \norm{\nstate} \to \infty$ in probability as $\norm{\state} \to \infty$.
\end{proposition}

\begin{proof}
  By \eqref{eq:reversible-cond-2}, we must have $\rpropcov D = \idmat - \rpropmat^2$. Since $\rpropmat = \rpropmat^\top$, this
  entails $\rpropcov D = (\rpropcov D)^\top = D \rpropcov$ and hence $\rpropcov D \succeq 0$ since $D, \rpropcov \succeq 0$. Thus $-\rpropcov D \preceq 0$ and
  \[
    \rpropmat^2 = \idmat - \rpropcov D \preceq \idmat.
  \]
  Therefore each eigenvalue $\sigma$ of $\rpropmat$ must have $\abs{\sigma} \leq 1$, since $\sigma^2$ is an eigenvalue of $\rpropmat^2$. But $\rpropmat$ is diagonalisable since it is symmetric, and hence $\opnorm{\rpropmat} \leq 1$.

  If $D \succ 0$ is strict, then the above matrix inequalities become strict also, and it follows that each $\abs{\sigma} < 1$ and hence $\opnorm{\rpropmat} < 1$. In this case, suppose $\nstate \sim \prop(\state, \cdot)$, and fix $K > 0$ arbitrarily. Let $\epsilon > 0$, and choose $L > 0$ large enough that
  \[
    \P(\nstate \in B(\rpropmat\state + \rpropvec, L)) > 1 - \epsilon.
  \]
  As $\norm{\state} \to \infty$,
  \[
    \norm{\state} - \norm{\rpropmat\state + \rpropvec} \geq \norm{\state}(1 - \opnorm{\rpropmat}) + \norm{\rpropvec} \to \infty
  \]
  since $1 - \opnorm{\rpropmat} > 0$, so if $\nstate \in B(\rpropmat\state + \rpropvec, L)$, then $\norm{\state} -
  \norm{\nstate} \to \infty$ also. Thus
  \[
    \P(\norm{\state} - \norm{\nstate} > K) > 1 - \epsilon
  \]
  for all $\norm{\state}$ large enough. Taking $\epsilon \to 0$ gives the result.
\end{proof}

In practice the assumption $D \succeq 0$ makes sense, since $\modestate$ is chosen near a minimum of $U$ and since $D$ is the Hessian of $\approxpot_k \approx U$ for $k = 1, 2$. Likewise, all sensible proposals (certainly including pCN) that we have found are such that $\rpropmat$ is symmetric, though we acknowledge the possibility that it may be desirable to violate this in some cases.

\section{Performance Gains} \label{sec:performance-gains}

\begin{lemma} \label{lem:cond-expect-orders}
    Suppose that $0 \leq X_n \in L^p$ and $\mathcal{F}_n$ is some $\sigma$-algebra for every $n \in \Z_{\geq 1}$. If $\E[X^p_n|\mathcal{F}_n] = O_\P(a_n)$, then $\E[X^\ell_n|\mathcal{F}_n] = O_\P(a_n^{\ell/p})$ for all $1 \leq \ell \leq p$. If moreover $0 \leq Y_n \in L^p$ gives $\E[Y^p_n|\mathcal{F}_n] = O_\P(a_n)$, then $\E[(X_n+Y_n)^p|\mathcal{F}_n] = O_\P(a_n)$.
\end{lemma}

\begin{proof}
    The first part is just Jensen's inequality:
    \[
        \E[X_n^\ell|\mathcal{F}_n] \leq \E[X_n^p|\mathcal{F}_n]^{\ell/p} = O_\Ptrue(a_n)^{\ell/p} = O_\Ptrue(a_n^{\ell/p}).
    \]
    The second part follows from the $C_p$-inequality, which gives
    \[
         \E[(X_n+Y_n)^p|\mathcal{F}_n] \leq 2^{p-1} \left(\E[X_n^p|\mathcal{F}_n]+\E[Y_n^p|\mathcal{F}_n]\right)
         = 2^{p-1}(O_\P(a_n) + O_\P(a_n)) = O_\P(a_n).
    \]
\end{proof}

\begin{theorem}
    Suppose each $\pot_i$ is $(k+1)$-times continuously differentiable, each $\potgradub_{k+1,i} \in L^{k+2}$, and $\E[\sum_{i=1}^{m^{(n)}} \potgradub_{k+1,i}|\Data_{1:n}] = O_\Ptrue(n)$. Likewise, assume each of $\norm{\state^{(n)} - \mapstate^{(n)}}$, $\norm{\state^{(n)} - \nstate^{(n)}}$, and $\norm{\modestate^{(n)} - \mapstate^{(n)}}$ is in $L^{k+2}$, and each of $\E[\norm{\state^{(n)} - \mapstate^{(n)}}^{k+1}|\Data_{1:n}]$, $\E[\norm{\state^{(n)} - \nstate^{(n)}}^{k+1}|\Data_{1:n}]$, and $\norm{\modestate^{(n)} - \mapstate^{(n)}}^{k+1}$ is $O_\Ptrue(n^{-(k+1)/2})$ as $n \to \infty$. Then $\bound$ defined by \eqref{eq:smh-bound} satisfies
    \[
        \E[\bound(\state^{(n)}, \nstate^{(n)}) | \Data_{1:n}] = O_\Ptrue(n^{(1-k)/2}).
    \]
\end{theorem}

\begin{proof}
    Write
    \[
        \bound(\state^{(n)}, \nstate^{(n)}) =  \statebound(\state^{(n)}, \nstate^{(n)}) \sum_{i=1}^{m^{(n)}} \databound_i.
    \]
    with $\statebound$ and $\databound$ defined by \eqref{eq:smh-bound} also. Observe that
    \begin{eqnarray*}
        \statebound(\state^{(n)}, \nstate^{(n)})
            &=& \norm{\state^{(n)} - \modestate^{(n)}}_1^{k+1} + \norm{\nstate^{(n)} - \modestate^{(n)}}_1^{k+1} \\
            &\leq& (\norm{\state^{(n)} - \modestate^{(n)}}_1 + \norm{\nstate^{(n)} - \modestate^{(n)}}_1)^{k+1} \\
            &\leq& (\norm{\state^{(n)} - \mapstate^{(n)}}_1 + \norm{\mapstate^{(n)} - \modestate^{(n)}}_1 + \norm{\nstate^{(n)} - \state^{(n)}}_1 + \norm{\state^{(n)} - \mapstate^{(n)}}_1 + \norm{\mapstate^{(n)} - \modestate^{(n)}}_1)^{k+1} \\
            &\leq& c (\underbrace{\norm{\nstate^{(n)} - \state^{(n)}} + \norm{\state^{(n)} - \mapstate^{(n)}} + \norm{\mapstate^{(n)} - \modestate^{(n)}}}_{\in L^{k+2}})^{k+1}
    \end{eqnarray*}
    for some $c > 0$, by the triangle inequality and norm equivalence. We thus have $\statebound(\state^{(n)}, \nstate^{(n)}) \in L^{(k+2)/(k+1)}$ and
    \[
        \E[\statebound(\state^{(n)}, \nstate^{(n)})|\Data_{1:n}] = O_\Ptrue(n^{-(k+1)/2}).
    \]
    Likewise,
    \[
        \sum_{i=1}^{m^{(n)}} \databound_i = \frac{1}{(k+1)!} \sum_{i=1}^{m^{(n)}} \potgradub_{k+1,i} \in L^{k+2}.
    \]
    Together this gives $\bound(\state^{(n)}, \nstate^{(n)}) \in L^1$ by H\"older's inequality. Since in our setup $(\state^{(n)}, \nstate^{(n)})$ is conditionally independent of all other randomness given $\Data_{1:n}$, we thus have
    \begin{equation} \label{eq:lambda_order}
        \E[\bound(\state^{(n)}, \nstate^{(n)})|\Data_{1:n}] = \E[\statebound(\state^{(n)}, \nstate^{(n)})|\Data_{1:n}] \E[\sum_{i=1}^{m^{(n)}} \databound_i|\Data_{1:n}]
        = O_\Ptrue(n^{(1-k)/2}).
    \end{equation}
\end{proof}

Note that in the preceding result we could use weaker integrability assumptions on $\norm{\state^{(n)} - \mapstate^{(n)}}$, $\norm{\state^{(n)} - \nstate^{(n)}}$, and $\norm{\modestate^{(n)} - \mapstate^{(n)}}$ by using a stronger integrability assumption on $\potgradub_{k+1,i}$. Most generally, for any $\epsilon \geq 0$ we could require each
\begin{eqnarray*}
    \potgradub_{k+1,i} &\in& L^{(k+1+\epsilon)/\epsilon} \\
    \norm{\state^{(n)} - \mapstate^{(n)}}, \norm{\state^{(n)} - \nstate^{(n)}}, \norm{\modestate^{(n)} - \mapstate^{(n)}} &\in& L^{k+1+\epsilon}.
\end{eqnarray*}
The case $\epsilon = 0$ would mean $\potgradub_{k+1,i} \in L^\infty$.

\begin{lemma} \label{lem:norm-grad-modestate}
    Suppose each $\pot_i$ is twice continuously differentiable, each $\potgradub_{2,i} \in L^3$, and $\sum_{i=1}^{m^{(n)}} \potgradub_{2,i} = O_\Ptrue(n)$. If $\norm{\modestate^{(n)} - \mapstate^{(n)}} = O_\Ptrue(1/\sqrt{n})$, then $\norm{\nabla \pot^{(n)}(\modestate^{(n)})}$ is in $L^{3/2}$ and $O_\Ptrue(\sqrt{n})$.
\end{lemma}

\begin{proof}
  By norm equivalence the Hessian satisfies
  \[
    \opnorm{\nabla^2 \pot^{(n)}(\state)} \leq c \norm{\nabla^2 \pot^{(n)}(\state)}_1 \leq c \sum_{i=1}^{m^{(n)}} \potgradub_{2,i}
  \]
  for some $c > 0$ (where $\norm{\cdot}_1$ is understood to be applied as if $\nabla^2 \pot^{(n)}(\state)$ were a vector), which means $\nabla \pot^{(n)}$ is $(c \sum_{i=1}^{m^{(n)}} \potgradub_{2,i})$-Lipschitz. Thus
  \begin{eqnarray*}
    \norm{\nabla \pot^{(n)}(\modestate^{(n)})}
      &=& \norm{\nabla \pot^{(n)}(\modestate^{(n)}) - \nabla \pot^{(n)}(\mapstate^{(n)})} \\
      &\leq& c \underbrace{(\sum_{i=1}^{m^{(n)}} \potgradub_{2, i})}_{\in L^3} \underbrace{\norm{\modestate^{(n)} - \mapstate^{(n)}}}_{\in L^{k+2} \subseteq L^3}
  \end{eqnarray*}
  since $k \geq 1$. By Cauchy-Schwarz we have therefore $\norm{\nabla \pot^{(n)}(\modestate^{(n)})} \in L^{3/2}$.
  
  Similarly, since $\modestate^{(n)}$ and $\mapstate^{(n)}$ are functions of  $\Data_{1:n}$,
    \begin{eqnarray*}
    \norm{\nabla \pot^{(n)}(\modestate^{(n)})}
      &=& \E[\norm{\nabla \pot^{(n)}(\modestate^{(n)}) - \nabla \pot^{(n)}(\mapstate^{(n)})}|\Data_{1:n}] \\
      &\leq& \E[c (\sum_{i=1}^{m^{(n)}} \potgradub_{2, i}) \norm{\modestate^{(n)} - \mapstate^{(n)}}|\Data_{1:n}] \\
      &=& c\underbrace{\E[\sum_{i=1}^{m^{(n)}} \potgradub_{2, i}|\Data_{1:n}]}_{= O_\Ptrue(n)} \underbrace{\norm{\modestate^{(n)} - \mapstate^{(n)}}}_{= O_\Ptrue(1/\sqrt{n})} \\
      &=& O_\Ptrue(\sqrt{n}).
  \end{eqnarray*}
\end{proof}
    
\begin{theorem} \label{SUPP:thm:solitary-term-asymptotics}
    Suppose the assumptions of Theorem \ref{thm:bound-asymptotics} hold, and additionally that for $2 \leq \ell \leq k$, each $\potgradub_{\ell,i} \in L^{\ell+1}$, and $\E[\sum_{i=1}^{m^{(n)}} \potgradub_{\ell,i}|\Data_{1:n}] = O_\Ptrue(n)$. Then
    \[
        -\log (1 \wedge \frac{\approxtarget_k^{(n)}(\nstate^{(n)})}{\approxtarget_k^{(n)}(\state^{(n)})}) = O_\Ptrue(1)
    \]
    for all $k \geq 1$.
\end{theorem}

\begin{proof}
    It is useful to denote
    \begin{eqnarray*}
        \pot^{(n)}(\state) &:=& \sum_{i=1}^{m^{(n)}} \pot_i(\state) \\
        \approxpot_k^{(n)}(\state) &:=& \sum_{i=1}^{m^{(n)}} \approxpot_{k, i}(\state) = -\log(\approxtarget^{(n)}(\state)).
    \end{eqnarray*}
    Observe that
    \begin{equation} \label{SUPP:eq:log-solitary-term-ub}
        0 \leq -\log (1 \wedge \frac{\approxtarget^{(n)}(\nstate^{(n)})}{\approxtarget^{(n)}(\state^{(n)})}) \leq \abs{\approxpot_k^{(n)}(\nstate^{(n)}) - \approxpot_k^{(n)}(\state^{(n)})}.
    \end{equation}
  Now,
  \begin{equation} \label{SUPP:eq:log-solitary-term}
    \approxpot_k^{(n)}(\nstate^{(n)}) - \approxpot_k^{(n)}(\state^{(n)})
    = \inner{\nabla \pot^{(n)}(\modestate^{(n)})}{\nstate^{(n)} - \state^{(n)}}
      + \sum_{2 \leq \abs{\midx} \leq k} \frac{\partial^\midx \pot^{(n)}(\modestate^{(n)})}{\midx!} ((\nstate^{(n)} - \modestate^{(n)})^\midx - (\state^{(n)} - \modestate^{(n)})^\midx).
  \end{equation}
  For the first term here, Cauchy-Schwarz gives
  \begin{eqnarray*}
    \E[\abs{\inner{\nabla U^{(n)}(\modestate^{(n)})}{\nstate^{(n)} - \state^{(n)}}}|\Data_{1:n}]
    &\leq& \E[\underbrace{\norm{\nabla \pot^{(n)}(\modestate^{(n)})}}_{\in L^{3/2}} \underbrace{\norm{\nstate^{(n)} - \state^{(n)}}}_{\in L^{k+2} \subseteq L^3}|\Data_{1:n}] \\
    &=& \underbrace{\norm{\nabla \pot^{(n)}(\modestate^{(n)})}}_{=O_\Ptrue(\sqrt{n})} \underbrace{\E[\norm{\nstate^{(n)} - \state^{(n)}}|\Data_{1:n}]}_{=O_\Ptrue(1/\sqrt{n})} \\
    &=& O_\Ptrue(1).
  \end{eqnarray*}
  Integrability follows from Lemma \ref{lem:norm-grad-modestate} and H\"older's inequality, and the asymptotic statements from conditional independence, Lemma \ref{lem:norm-grad-modestate}, and Lemma \ref{lem:cond-expect-orders}. For the summation in \eqref{SUPP:eq:log-solitary-term}, note that
  \[
    \abs{\partial^\midx \pot^{(n)}(\modestate^{(n)})}
    \leq \sum_{i=1}^{m^{(n)}} \abs{\partial^\midx \pot_i(\modestate^{(n)})}
    \leq \sum_{i=1}^{m^{(n)}} \potgradub_{\abs{\midx}, i},
  \]
  and that for some $c > 0$,
  \begin{eqnarray*}
    \abs{(\nstate^{(n)} - \modestate^{(n)})^\midx - (\state^{(n)} - \modestate^{(n)})^\midx}
    &\leq& \norm{\nstate^{(n)} - \modestate^{(n)}}_\infty^{\abs{\midx}} + \norm{\state^{(n)} - \modestate^{(n)}}_\infty^{\abs{\midx}} \\
      &\leq& c\norm{\nstate^{(n)} - \modestate^{(n)}}^{\abs{\midx}} + c\norm{\state^{(n)} - \modestate^{(n)}}^{\abs{\midx}}
  \end{eqnarray*}
  by norm equivalence. Thus, conditional on $Y_{1:n}$, the absolute value of the summation in \eqref{SUPP:eq:log-solitary-term} is bounded above by
  \begin{eqnarray*}
    && \sum_{2 \leq \abs{\midx} \leq k} \frac{1}{\midx!} \E[\underbrace{(\sum_{i=1}^{m^{(n)}} \potgradub_{\abs{\midx}, i})}_{\in L^{\abs{\midx}+1}} (c\underbrace{\norm{\nstate^{(n)} - \modestate^{(n)}}^{\abs{\midx}}}_{\in L^{(\abs{\midx}+1)/\abs{\midx}}} + c\underbrace{\norm{\state^{(n)} - \modestate^{(n)}}^{\abs{\midx}}}_{\in L^{(\abs{\midx}+1)/\abs{\midx}}} | \Data_{1:n}] \\
    &=& \sum_{2 \leq \abs{\midx} \leq k} \frac{c}{\midx!} \underbrace{\E[\sum_{i=1}^{m^{(n)}} \potgradub_{\abs{\midx}, i}|\Data_{1:n}]}_{=O_\Ptrue(n)}
    (\underbrace{\E[\norm{\nstate^{(n)} - \modestate^{(n)}}^{\abs{\midx}}|\Data_{1:n}]}_{=O_\Ptrue(n^{-\abs{\midx}/2})} + \underbrace{\E[\norm{\state^{(n)} - \modestate^{(n)}}^{\abs{\midx}} | \Data_{1:n}]}_{=O_\Ptrue(n^{-\abs{\midx}/2})}) \\
    &=& O_\Ptrue(1).
  \end{eqnarray*}
  Again, integrability follows from H\"older's inequality. The second line holds since $\modestate^{(n)} \equiv \modestate^{(n)}(\Data_{1:n})$ and since $(\state^{(n)},\nstate^{(n)})$ is conditionally independent of all other randomness given $\Data_{1:n}$. Finally, the asymptotics follow from the law of large numbers and Lemma \ref{lem:cond-expect-orders} (noting that each $\abs{\midx} \geq 2$).
  
  Inspection of \eqref{SUPP:eq:log-solitary-term} now shows that \eqref{SUPP:eq:log-solitary-term} is $O_\Ptrue(1)$ as required.
\end{proof}

\subsection{Sufficient Conditions}

We are interested in sufficient conditions that guarantee the convergence rate assumptions in Theorem \ref{thm:bound-asymptotics} will hold. For simplicity we assume throughout that the likelihood of a data point $p(\data|\state)$ admits a density w.r.t. Lebesgue measure and that $\Ptrue$ also admits a Lebesgue density denoted $\ptrue(\data)$.

\subsubsection{Concentration Around the Mode} \label{sec:concentration-around-mode}

We first consider the assumption
\[
     \E[\norm{\state^{(n)} - \mapstate^{(n)}}^{k+1}|\Data_{1:n}] = O_{\Ptrue}(n^{-(k+1)/2}).
\]
Intuitively, this says that the distance of $\state^{(n)}$ from the mode is $O(1/\sqrt{n})$, and hence connects directly with standard concentration results on Bayesian posteriors. To establish this rigorously, it is enough to show that for some $\convstate \in \statespace$ both
\begin{eqnarray}
    \E[\norm{\state^{(n)} - \convstate}^{k+1}|\Data_{1:n}] &=& O_{\Ptrue}(n^{-(k+1)/2}) \label{eq:concentration-around-conv} \\
    \E[\norm{\mapstate^{(n)} - \convstate}^{k+1}|\Data_{1:n}] &=& O_{\Ptrue}(n^{-(k+1)/2}), \label{eq:map-dist-from-conv-alt}
\end{eqnarray}
which entails the result by Lemma \ref{lem:cond-expect-orders} and the triangle inequality. Note that $\mapstate^{(n)} \equiv \mapstate^{(n)}(\Data_{1:n})$ is deterministic function of the data, so that \eqref{eq:map-dist-from-conv-alt} may be written more simply as
\begin{equation} \label{eq:map-dist-from-conv}
    \sqrt{n}(\mapstate^{(n)} - \convstate) = O_\Ptrue(1).
\end{equation}
We give sufficient conditions for \eqref{eq:concentration-around-conv} and \eqref{eq:map-dist-from-conv} now.

By Proposition \ref{prop:concentration-around-conv} below, \eqref{eq:concentration-around-conv} holds as soon as we show that
\begin{equation} \label{eq:tail-integral-around-conv-negligible}
    \E[\norm{\sqrt{n}(\state^{(n)} - \convstate)}^{k+1} \ind(\norm{\sqrt{n}(\state^{(n)} - \convstate)} > M_n)|\Data_{1:n}] \Pto{\Ptrue} 0, \quad\quad \text{for all $M_n \to \infty$}.
\end{equation}
This condition is a consequence of standard assumptions used to prove the Bernstein-von Mises theorem (BvM): in particular, it is \citep[(10.9)]{vaart1998asymptotic} when the model is well-specified (i.e. $\ptrue = p(\data|\truestate)$ for some $\truestate \in \state$), and \citep[(2.16)]{kleijn2012bernstein} in the misspecified case. In both cases
\[
    \convstate = \argmin_{\state \in \statespace} \KL{\ptrue(\data)}{p(\data|\state)},
\]
where $\KL{\cdot}{\cdot}$ denotes the Kullback-Leibler divergence. The key assumption required for \eqref{eq:tail-integral-around-conv-negligible} is then the existence of certain test sequences $\phi_n \equiv \phi_n(\Data_{1:n})$ with $0 \leq \phi_n \leq 1$ such that, whenever $\epsilon > 0$, both
\begin{equation} \label{eq:test-sequences}
    \int \phi_n(\data_{1:n}) \prod_{i=1}^n \ptrue(\data_i) d\data_{1:n} \to 0
    \quad\quad\text{and}\quad\quad
    \sup_{\norm{\state - \convstate} \geq \epsilon} \int (1 - \phi_n(\data_{1:n})) \prod_{i=1}^n \frac{p(\data_i|\state)}{p(\data_i|\convstate)} \ptrue(\data_i) d\data_{1:n} \to 0,
\end{equation}
Note that in the well-specified case these conditions say that $\phi_n$ is uniformly consistent for testing the hypothesis $H_0 : \state = \truestate$ versus $H_1 : \norm{\state - \truestate} \geq \epsilon$. Since $\phi_n$ may have arbitrary form, this requirement does not seem arduous. Sufficient conditions are given by \citep[Lemma 10.4, Lemma 10.6]{vaart1998asymptotic} for the well-specified case, and \citep[Theorem 3.2]{kleijn2012bernstein} for the misspecified case.

In addition to \eqref{eq:test-sequences}, we require in both the well-specified and misspecified cases that the prior $p(\state)$ be continuous and positive at $\convstate$ and satisfy
\[
    \int \norm{\state}^{k+1} p(\state) d\state < \infty.
\]
There are additionally some mild smoothness and regularity conditions imposed on the likelihood, which are naturally stronger in the misspecified case than in the well-specified one. In the well-specified case we require $p(\data|\state)$ is differentiable in quadratic mean at $\convstate$ \citep[(7.1)]{vaart1998asymptotic}. In the misspecified case the conditions are more complicated. We omit repeating these for brevity and instead refer the reader to the statements of Lemma 2.1 and Theorem 3.1 in \citep{kleijn2012bernstein}.

\begin{lemma} \label{lem:OpMn-implies-Op1}
    Suppose a sequence of random variables $X_n$ is $O_\P(M_n)$ for every sequence $M_n \to \infty$. Then $X_n = O_\P(1)$.
\end{lemma}

\begin{proof}
    Suppose $X_n \neq O_\P(1)$. Then, for some $\epsilon > 0$, for every $c > 0$ we have $\P(\abs{X_n} > c) \geq \epsilon$ for infinitely many $X_n$. This allows us to choose a subsequence $X_{n_k}$ such that $\P(\abs{X_{n_k}} > k) \geq \epsilon$ for each $k \in \Z_{\geq 1}$. Let
    \[
        M_n := \begin{cases}
            k & \text{if $n = n_k$ for some (necessarily unique) $k$} \\
            n & \text{otherwise}.
        \end{cases}
    \]
    Then $M_n \to \infty$ but $\P(\abs{X_n} > M_n) \geq \epsilon$ occurs for infinitely many $n$ and hence $X_n \neq O_\P(M_n)$.
\end{proof}

\begin{proposition} \label{prop:concentration-around-conv}
    Suppose that for some $\convstate \in \statespace$ and $\ell \geq 0$,
    \[
        \E[\norm{\sqrt{n}(\state^{(n)} - \convstate)}^\ell \ind(\norm{\sqrt{n}(\state^{(n)} - \convstate)} > M_n)|\Data_{1:n}] \Pto{\Ptrue} 0
    \]
    whenever $M_n \to \infty$. Then
    \[
        \E[\norm{\state^{(n)} - \convstate}^\ell|\Data_{1:n}] = O_{\Ptrue}(n^{-\ell/2}).
    \]
\end{proposition}

\begin{proof}
    For $M_n \to \infty$, our assumption lets us write
    \begin{eqnarray*}
        n^{\ell/2} \E[\norm{\state^{(n)} - \convstate}^\ell|\Data_{1:n}]
            &=& \E[\norm{\sqrt{n}(\state^{(n)} - \convstate)}^\ell \ind(\norm{\sqrt{n}(\state^{(n)} - \convstate)} \leq M_n)|\Data_{1:n}] + o_\Ptrue(1) \\
            &\leq& M_n^\ell + o_\Ptrue(1) \\
            &=& O_\Ptrue(M_n^\ell).
    \end{eqnarray*}
    Since $M_n$ was arbitrary, Lemma \ref{lem:OpMn-implies-Op1} entails the left-hand side is $O_\Ptrue(1)$, so that
    \[
        \E[\norm{\state^{(n)} - \convstate}^\ell|\Data_{1:n}] = O_\Ptrue(n^{-\ell/2}).
    \]
\end{proof}

It remains to give conditions for \eqref{eq:map-dist-from-conv}. Our discussion here is fairly standard. Recall that for $\mlestate^{(n)}$ the maximum likelihood estimator,
\[
    \sqrt{n}(\mlestate^{(n)} - \convstate) = O_\Ptrue(1)
\]
often holds under under mild smoothness assumptions. We show here that effectively those same assumptions are also sufficient to guarantee a similar result for $\mapstate^{(n)}$.

In the following we define
\[
    \Lobj_n(\state) := \frac{1}{n} \sum_{i=1}^n \log p(\Data_i|\state).
\]
Note that by definition
\[
    \mlestate^{(n)} = \sup_{\state \in \statespace} \Lobj_n(\state).
\]
Our first result here shows that if both the MAP and the MLE are consistent and the prior is well-behaved, then the MAP is a \emph{near maximiser} of $\Lobj_n$ in the sense that \eqref{eq:near-optimizer}. Combined with mild smoothness assumptions on the likelihood, \eqref{eq:near-optimizer} is a standard condition used to show results such as \eqref{eq:map-dist-from-conv}. See for instance \citep[Theorem 5.23]{vaart1998asymptotic} for a detailed statement.

\begin{proposition}
    Suppose for some $\convstate \in \statespace$ that $\mapstate^{(n)}, \mlestate^{(n)} \Pto{\Ptrue} \convstate$ and that the prior $p(\state)$ is continuous and positive at $\convstate$, then
    \begin{equation} \label{eq:near-optimizer}
        \Lobj_n(\mapstate^{(n)}) \geq \Lobj_n(\mlestate^{(n)}) - o_\Ptrue(1/n).
    \end{equation}
\end{proposition}

\begin{proof}
    Observe that by definition of the MAP,
    \[
        \Lobj_n(\mlestate^{(n)}) + \frac{1}{n} \log p(\mlestate^{(n)})
            \leq \Lobj_n(\mapstate^{(n)}) + \frac{1}{n} \log p(\mapstate^{(n)}).
    \]
    We can rewrite this inequality as
    \[
        \Lobj_n(\mapstate^{(n)}) \geq \Lobj_n(\mlestate^{(n)}) + \frac{1}{n} \log \frac{p(\mlestate^{(n)})}{p(\mapstate^{(n)})}.
    \]
    The second term on the right-hand side is $o_\Ptrue(1/n)$, since our assumption on the prior gives
    \[
        \frac{p(\mlestate^{(n)})}{p(\mapstate^{(n)})} \Pto{\Ptrue} 1.
    \]
\end{proof}

We next consider how to show that the MAP is indeed consistent, as the vast majority of such results in this area only consider the MLE. However, assuming the prior is not pathological, arguments for the consistency of the MLE ought to apply also for the MAP, since the MAP optimises the objective function
\[
    \Lobj_n(\state) + \frac{1}{n} \log p(\state),
\]
which is asymptotically equivalent to $\Lobj_n(\state)$ as $n \to \infty$ whenever $p(\state) > 0$. By way of example, we show that \citep[Theorem 5.7]{vaart1998asymptotic}, which can be used to show the consistency of the MLE, also applies to the MAP. For this, we assume that
\begin{equation} \label{eq:log-likelihood-integrability}
    \int \abs{\log p(\data|\convstate)} \ptrue(\data) d\data < \infty,
\end{equation}
and define
\[
    \Lobj(\state) := \int \log p(\data|\state) \ptrue(\data) d\data.
\]

\begin{proposition}
    Suppose that \eqref{eq:log-likelihood-integrability} holds, that
    \[
        \sup_{\state \in \statespace} \abs{\Lobj_n(\state) - \Lobj(\state)} \Pto{\Ptrue} 0,
    \]
    and that for some $\epsilon > 0$ and $\convstate \in \statespace$
    \begin{equation} \label{eq:map-consistent-cond}
        \sup_{\norm{\state - \convstate} \geq \epsilon} \Lobj(\state) < \Lobj(\convstate).
    \end{equation}
    Further, suppose the prior $p(\state)$ is continuous and positive at $\convstate$, and that
    $\sup_{\state \in \statespace} p(\state) < \infty$. Then both $\mlestate^{(n)}, \mapstate^{(n)} \Pto{\Ptrue} \convstate$.
\end{proposition}

\begin{proof}
    For each $\state \in \statespace$ we have $\Lobj_n(\state) \Pto{\Ptrue} \Lobj(\state)$ as $n \to \infty$ by the law of large numbers, and thus $\mlestate^{(n)} \Pto{\Ptrue} \convstate$ by \citep[Theorem 5.7]{vaart1998asymptotic}. Since $p(\state)$ is continuous and positive at $\convstate$, this yields that
    \begin{equation} \label{eq:prior-mle-lb}
        \Ptrue(p(\mlestate^{(n)}) > c) \to 1
    \end{equation}
    for some $c > 0$, as well as
    \[
        \frac{1}{n} \log p(\mlestate^{(n)}) = O_\Ptrue(1/n).
    \]
    Now, by maximality
    \[
        \Lobj_n(\mlestate^{(n)}) + \frac{1}{n} \log p(\mlestate^{(n)})
            \leq \Lobj_n(\mapstate^{(n)}) + \frac{1}{n} \log p(\mapstate^{(n)})
            \leq \Lobj_n(\mlestate^{(n)}) + \frac{1}{n} \log p(\mapstate^{(n)}).
    \]
    Observe that it implies that $p(\mlestate^{(n)}) \leq p(\mapstate^{(n)})$. Together with \eqref{eq:prior-mle-lb} and our boundedness assumption on the prior, this gives
    \[
        \frac{1}{n} \log p(\mapstate^{(n)}) = O_\Ptrue(1/n).
    \]
    We can thus write
    \[
        \Lobj_n(\mapstate^{(n)}) \geq \Lobj_n(\mlestate^{(n)}) + O_\Ptrue(1/n).
    \]
    The result now follows from \citep[Theorem 5.7]{vaart1998asymptotic}.
\end{proof}

Observe that by negating \eqref{eq:map-consistent-cond} and adding the constant $\int \ptrue(\data) \log \ptrue(\data) d\data$ to both sides, we see it is equivalent to the perhaps more intuitive condition
\[
    \inf_{\norm{\state - \convstate} \geq \epsilon} \KL{\ptrue(\data)}{p(\data|\state)}
        >  \KL{\ptrue(\data)}{p(\data|\convstate)}.
\]

\subsubsection{Scaling of the Proposal} \label{sec:prop-scale}

We now consider the assumption
\begin{equation} \label{eq:proposal-concentration}
    \E[\norm{\state^{(n)} - \nstate^{(n)}}^{k+1}|\Data_{1:n}] = O_\Ptrue(n^{-(k+1)/2}).
\end{equation}
Intuitively this holds if we scale our proposal like $1/\sqrt{n}$. We consider here proposals based on a noise distribution $\xi^{(n)} \iid \Normal(0, \idmat)$, but generalisations are possible. We immediately obtain \eqref{eq:proposal-concentration} for instance with the scaled random walk proposal \eqref{eq:scaled-random-walk-prop}, for which
\[
    \nstate^{(n)} = \state^{(n)} + \frac{\sigma}{\sqrt{n}} \xi^{(n)}.
\]
Similarly, the $\approxtarget_1$-reversible proposal defined by \eqref{eq:scaled-first-order-prop} has
\[
    \nstate^{(n)} = \state^{(n)} - \frac{1}{2n} \nabla \pot^{(n)}(\modestate^{(n)}) + \frac{\sigma}{\sqrt{n}} \xi^{(n)},
\]
with $\xi^{(n)} \iid \Normal(0, \idmat)$. If the conditions of Lemma \ref{lem:norm-grad-modestate} hold, then the second term is $O_\Ptrue(1/\sqrt{n})$ and \eqref{eq:proposal-concentration} follows.

More generally we can consider trying to match the covariance of our noise to the covariance of our target. Intuitively, under usual circumstances, $[\nabla^2 \pot^{(n)}(\modestate^{(n)})]^{-1}$ is approximately proportional to the inverse observed Fisher information at $\convstate$, and hence preconditioning $\xi^{(n)}$ by $S^{(n)}$ such that
\[
    S^{(n)} {S^{(n)}}^\top = [\nabla^2 \pot^{(n)}(\modestate^{(n)})]^{-1}
\]
matches our proposal to the characteristics of the target. Such an $S^{(n)}$ can be computed for instance via a Cholesky decomposition.

Under usual circumstances this achieves a correctly scaled proposal. In particular, if
\begin{eqnarray}
    \modestate^{(n)} &\Pto{\Ptrue}& \convstate \label{eq:modestate-consistent} \\
    \frac{1}{n} \partial_j \partial_k \pot^{(n)}(\convstate) &\Pto{\Ptrue}& \mathcal{I}_{j,k} \label{eq:average-potentials-converge}
\end{eqnarray}
for some constants $\mathcal{I}_{j,k}$, then Proposition \ref{prop:observed-information-root} below entails $\opnorm{{S^{(n)}}} = O_\Ptrue(1/\sqrt{n})$. Thus \eqref{eq:proposal-concentration} holds for the preconditioned random walk proposal \eqref{proposalpreconditionedRW} for which
\[
    \nstate^{(n)} = \state^{(n)} + S^{(n)} \xi^{(n)},
\]
since
\begin{equation} \label{eq:precond-noise-bound}
    \norm{S^{(n)} \xi^{(n)}} \leq \opnorm{S^{(n)}} \norm{\xi^{(n)}} = O_\Ptrue(1/\sqrt{n}). 
\end{equation}
The same is also true a pCN proposal. In this case
\[
    \nstate^{(n)} - \state^{(n)} = (\sqrt{\rho} - 1)(\state^{(n)} - \modestate^{(n)}) + (\sqrt{\rho} - 1)([\nabla^2 \pot^{(n)}(\modestate^{(n)})]^{-1} \nabla \pot^{(n)}(\modestate^{(n)})) + \sqrt{1-\rho} S^{(n)} \xi^{(n)}.
\]
Note that here the first term satisfies
\[
    E[\norm{\state^{(n)} - \modestate^{(n)}}^3|\Data_{1:n}] = O_\Ptrue(n^{-3/2}),
\]
while the remaining two terms are $O_\Ptrue(1/\sqrt{n})$ by Lemma \ref{lem:norm-grad-modestate} and \eqref{eq:precond-noise-bound}. This gives \ref{eq:proposal-concentration} by Lemma \ref{lem:cond-expect-orders}.

Condition \eqref{eq:modestate-consistent} holds for instance under the assumptions of Theorem \ref{thm:bound-asymptotics} and provided concentration around $\convstate$ of the kind described in Section \ref{sec:concentration-around-mode} occurs. Condition \eqref{eq:average-potentials-converge} will also often hold in practice. For instance, if
\[
    \pot^{(n)}(\state) = - \log p(\state) -\sum_{i=1}^n \log p(\Data_i|\state),
\]
and if the prior is positive at $\convstate$, then for all $1 \leq j, k \leq d$ the law of large numbers gives
\begin{eqnarray*}
    \frac{1}{n} \partial_j \partial_k \pot^{(n)}(\convstate) &=&  -\frac{1}{n} \partial_j \partial_k \log p(\convstate) - \frac{1}{n} \sum_{i=1}^n \partial_j \partial_k \log p(\Data_i|\convstate)\\
        &\Pto{\Ptrue}& -\int \partial_j \partial_k \log p(\data|\convstate) \ptrue(\data) d\data
\end{eqnarray*}
when the derivatives and the integral exists. More generally our model may be specified conditional on i.i.d. covariates $X_i$ so that
\[
    \pot^{(n)}(\state) = - \log p(\state) -\sum_{i=1}^n \log p(\Data_i|\state, X_i) + \log p(X_i),
\]
in which case the same argument still applies. (Note that here abuse notation by considering our data $\Data_i \equiv (X_i, \Data_i)$, where the right-hand $\Data_i$ are response variables.)

\begin{proposition} \label{prop:observed-information-root}
  Suppose for some $\convstate \in \statespace$ we have $\modestate^{(n)} \Pto{\Ptrue} \convstate$ and
  \[
    \frac{1}{n} \partial_j \partial_k \pot^{(n)}(\convstate) \Pto{\Ptrue} \mathcal{I}_{jk}
  \]
  for all $1 \leq j,k \leq d$. Suppose moreover that each $\potgradub_{3,i} \in L^1$ and each $\nabla^2 \pot^{(n)}(\modestate^{(n)}) \succ 0$. If $[\nabla^2 \pot^{(n)}(\modestate^{(n)})]^{-1} = S^{(n)} {S^{(n)}}^\top$ for some $S^{(n)} \in \R^{d\times d}$, then
  \[
    \opnorm{S^{(n)}} = O_\Ptrue(1/\sqrt{n}).
  \]
\end{proposition}

\begin{proof}
    Suppose $\abs{\midx} = 2$. Note that since for each $i$ and $\state$
    \[
        \norm{\nabla \partial^\midx \pot_i(\state)}
            \leq c \norm{\nabla \partial^\midx \pot_i(\state)}_1
            = c \sum_{j=1}^d \abs{\partial_j \partial^\midx U_i(\state)}
            \leq c d \potgradub_{3,i},
    \]
    for some $c > 0$ by norm equivalence, it follows that $\partial^\midx U_i$ is $cd\potgradub_{3,i}$-Lipschitz. Consequently for each $\state$
    \[
        \Abs{\frac{1}{n} \partial^\midx \pot^{(n)}(\state) - \frac{1}{n} \partial^\midx \pot^{(n)}(\convstate)}
            \leq \frac{1}{n} \sum_{i=1}^{m^{(n)}} \abs{\partial^\midx \pot_i(\state) - \partial^\midx \pot_i(\convstate)}
        \leq \frac{1}{n} (\sum_{i=1}^{m^{(n)}} \potgradub_{3,i}) cd\norm{\state - \convstate}.
    \]
    Thus given $K, \eta > 0$
    \begin{eqnarray*}
        \P\left(\sup_{\norm{\state - \convstate} < K} \Abs{\frac{1}{n} \partial^\midx \pot^{(n)}(\state) - \frac{1}{n} \partial^\midx \pot^{(n)}(\convstate)} > \eta \right)
        &\leq& \P\left(\frac{1}{n} \sum_{i=1}^{m^{(n)}} \potgradub_{3,i} > \eta c^{-1}d^{-1} K^{-1}\right), \\
        &\leq& \frac{\E[\potgradub_{3,i}]}{\eta c^{-1}d^{-1}K^{-1}},
    \end{eqnarray*}
    by Markov's inequality. It is clear that given any $\eta > 0$ the right-hand side can be made arbitrarily small by taking $K \to 0$, which yields $n^{-1} \partial^{\midx} \pot^{(n)}(\state)$ is stochastic equicontinuous at $\convstate$, and consequently that
    \[
        \frac{1}{n} \partial^{\midx} \pot^{(n)}(\modestate^{(n)}) - \frac{1}{n} \partial^{\midx} \pot^{(n)}(\convstate) \Pto{\Ptrue} 0,
    \]
    see \citep[page 139]{pollard2012convergence}.
    
    Define the matrix $\mathcal{I} \in \R^{d\times d}$ by the constants $\mathcal{I}_{jk}$. We thus have
    \[
        \frac{1}{n} \nabla^2 \pot^{(n)}(\modestate^{(n)}) \Pto{\Ptrue} \mathcal{I}
    \]
    since it converges element-wise. Thus by the continuous mapping theorem
    \[
        n \opnorm{\nabla^2 \pot^{(n)}(\modestate^{(n)})^{-1}} \Pto{\Ptrue} \opnorm{\mathcal{I}^{-1}},
    \]
    from which it follows that
    \[
        \opnorm{[\nabla^2 \pot^{(n)}(\modestate^{(n)})]^{-1}} = O_\Ptrue(1/n).
    \]
    It is a standard result from linear algebra that
    \[
        \opnorm{[\nabla^2 \pot^{(n)}(\modestate^{(n)})]^{-1}} = \opnorm{S^{(n)}}^2,
    \]
    which gives the result.
\end{proof}

\section{Applications}\label{Supp:Applications}

We give here the results of applying our method to a logistic regression and a robust linear regression example. In both cases we write our covariates as $x_i$ and responses as $\data_i$, and our target is the posterior
\[
    \target(\state) = p(\state|x_{1:n}, \data_{1:n})
        \propto p(\state) \prod_{i=1}^n p(\data_i|\state, x_i).
\]

\subsection{Logistic Regression} \label{sec:logistic-regression-bounds}

In this case we have $x_i \in \R^d$, $\data_i \in \set{0, 1}$, and
\[
  p(y_i | \state, x_i) = \Bernoulli(y_i | \frac{1}{1 + \exp(-\state^\top x_i)}).
\]
For simplicity we assume a flat prior $p(\state) \equiv 1$, which allows factorising $\target$ like \eqref{eq:base-factorisation} with $m = n$ and $\basetarget_i(\state) = p(\data_i|\state, x_i)$. It is then easy to show that
\[
  \pot_i(\state) = -\log \widetilde{\target}_i(\state) = \log(1 + \exp(\state^T x_i)) - \data_i \state^\top x_i.
\]
We require upper bounds $\potgradub_{k+1,i}$ of the form \eqref{eq:grad-upper-bound} for these terms. For this we let $\sigma(z) = 1/(1+\exp(-z))$ and note the identity $\sigma'(z) = \sigma(z)(1 - \sigma(z))$, which entails
\[
    \partial_j \sigma(\state^\top x_i)
        = -x_{ij} (\sigma(\state^\top x_i) - \sigma(\state^\top x_i)^2).
\]
We then have
\begin{eqnarray*}
    \partial_j U_i(\state) &=& x_{ij}(\sigma(\state^\top x_i) - y_i) \\
    \partial_k \partial_j U_i(\state) &=& x_{ij} x_{ik} (\sigma(\state^\top x_i) - \sigma(\state^\top x_i)^2) \\
    \partial_\ell \partial_k \partial_j U_i(\state) &=& x_{ij} x_{ik} (x_{i\ell} (\sigma(\state^\top x_i) - \sigma(\state^\top x_i)^2) - 2 \sigma(\state^\top x_i) x_{i\ell} (\sigma(\state^\top x_i) - \sigma(\state^\top x_i)^2)) \\
    &=& x_{ij} x_{ik} x_{i\ell} (\sigma(\state^\top x_i) - \sigma(\state^\top x_i)^2) (1 -  2 \sigma(\state^\top x_i)).
\end{eqnarray*}
It is possible to show that (whether $y_i = 0$ or $y_1 = 1$)
\begin{eqnarray*}
    \sup_{t \in \R} \abs{\sigma(t) - y_i} &=& 1 \\
    \sup_{t \in \R} \abs{\sigma(t) - \sigma(t)^2} &=& \frac{1}{4} \\
    \sup_{t \in \R} \abs{(\sigma(t) - \sigma(t)^2)(1 -  2 \sigma(t))} &=& \frac{1}{6\sqrt{3}}.
\end{eqnarray*}
Thus setting
\begin{eqnarray*}
    \potgradub_{1,i} &:=& \max_{1 \leq j \leq d} \abs{x_{ij}} \\
    \potgradub_{2,i} &:=& \frac{1}{4} \max_{1 \leq j \leq d} \abs{x_{ij}}^2 \\
    \potgradub_{3,i} &:=& \frac{1}{6\sqrt{3}} \max_{1 \leq j \leq d} \abs{x_{ij}}^3
\end{eqnarray*}
satisfies \eqref{eq:grad-upper-bound}.

In Figure~\ref{fig:lr_histogram} we compare the histogram of the samples of the first coordinate $\theta_1$ to the marginal of the Gaussian approximation. This is done for $n=2048$, the smallest data size for which we saw a significant ESS improvement of SMH-2 over MH, and for larger $n$ showing the convergence of the Gaussian approximation and the posterior.

\begin{figure}[ht]

\begin{center}

\subfigure[$n=2048$]{
\includegraphics[width=0.5\columnwidth]{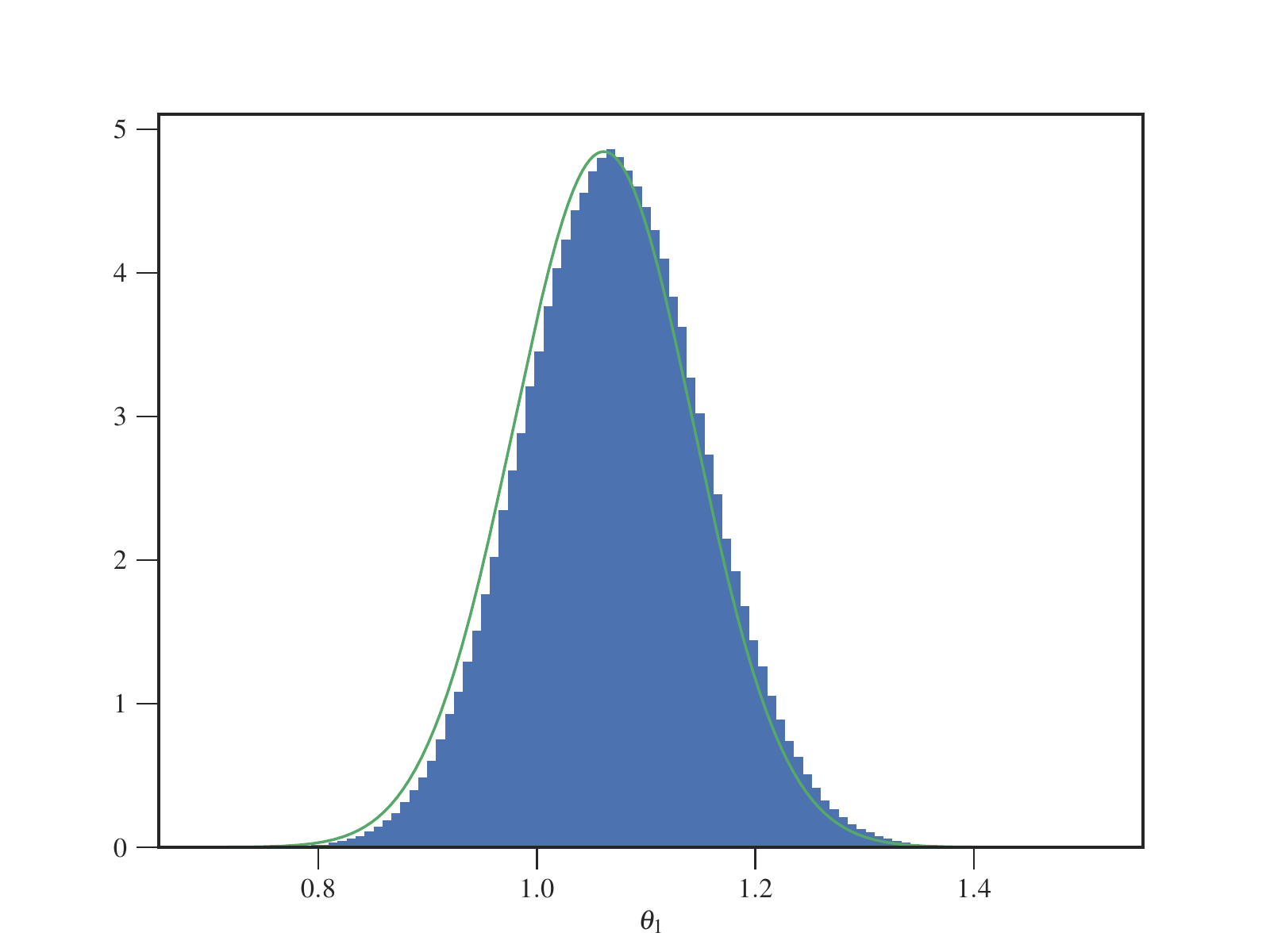}
}
\subfigure[$n=8192$]{
\includegraphics[width=0.5\columnwidth]{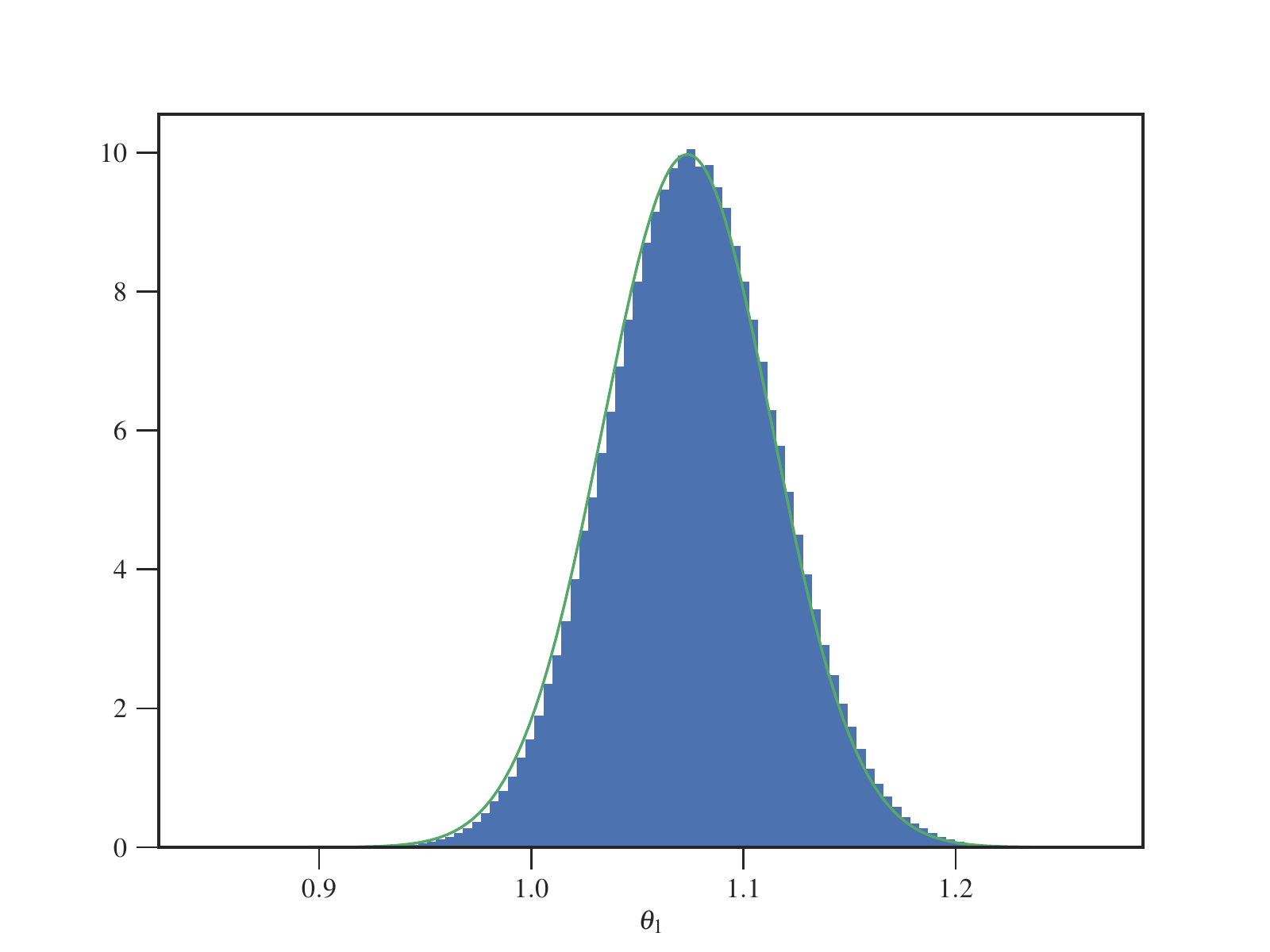}
%\caption{$N=8192$}
}
\subfigure[$n=65536$]{
\includegraphics[width=0.5\columnwidth]{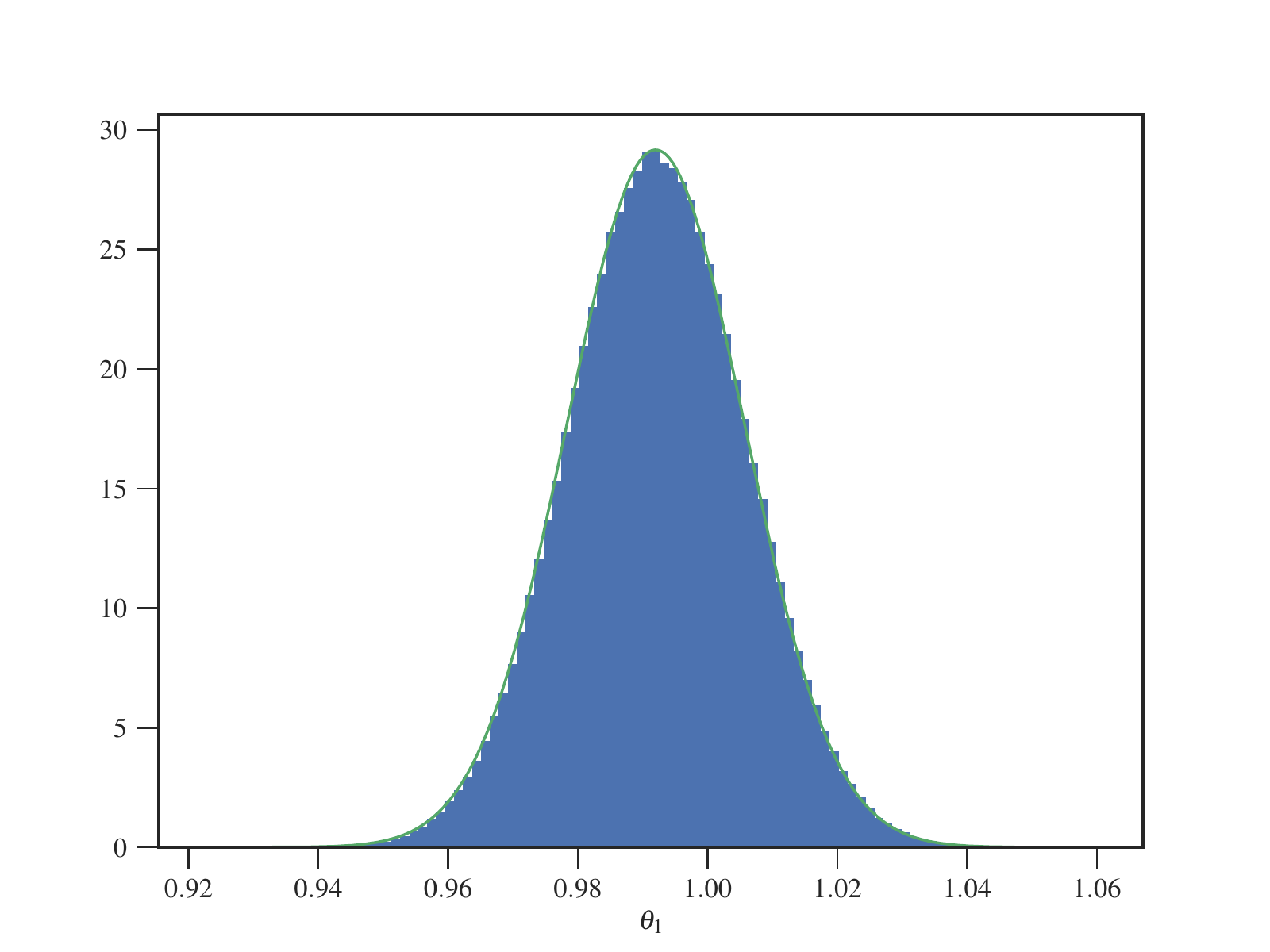}
%\caption{$N=65536$}
}

%\centerline{\includegraphics[width=0.3\columnwidth]{fig/histogram_model=lr_N=2048_d=10.pdf}}

%\centerline{\includegraphics[width=0.3\columnwidth]{fig/histogram_model=lr_N=8192_d=10.pdf}}

%\centerline{\includegraphics[width=0.3\columnwidth]{fig/histogram_model=lr_N=65536_d=10.pdf}}

\caption{Histogram of samples of first regression coefficient ($\theta_1$) versus marginal of Gaussian approximation (green lines).}
\label{fig:lr_histogram}

\end{center}

\end{figure}

In Figure~\ref{fig:rlr_nonreversible_iact_20d} we demonstrate the performance of the algorithm where $\theta$ is of dimension 20. The results are qualitatively similar to the 10-dimensional case in that SMH-2 eventually performs better than MH as the number of data increases. However for the 20-dimensional model SMH-2 yields superior performance to MH around the point at which $n$ exceeds 32768, whereas in the 10-dimensional model this happens for $n$ exceeding 2048.

\begin{figure}[ht]
\vskip 0.2in
\begin{center}
\centerline{\includegraphics[width=0.5\columnwidth]{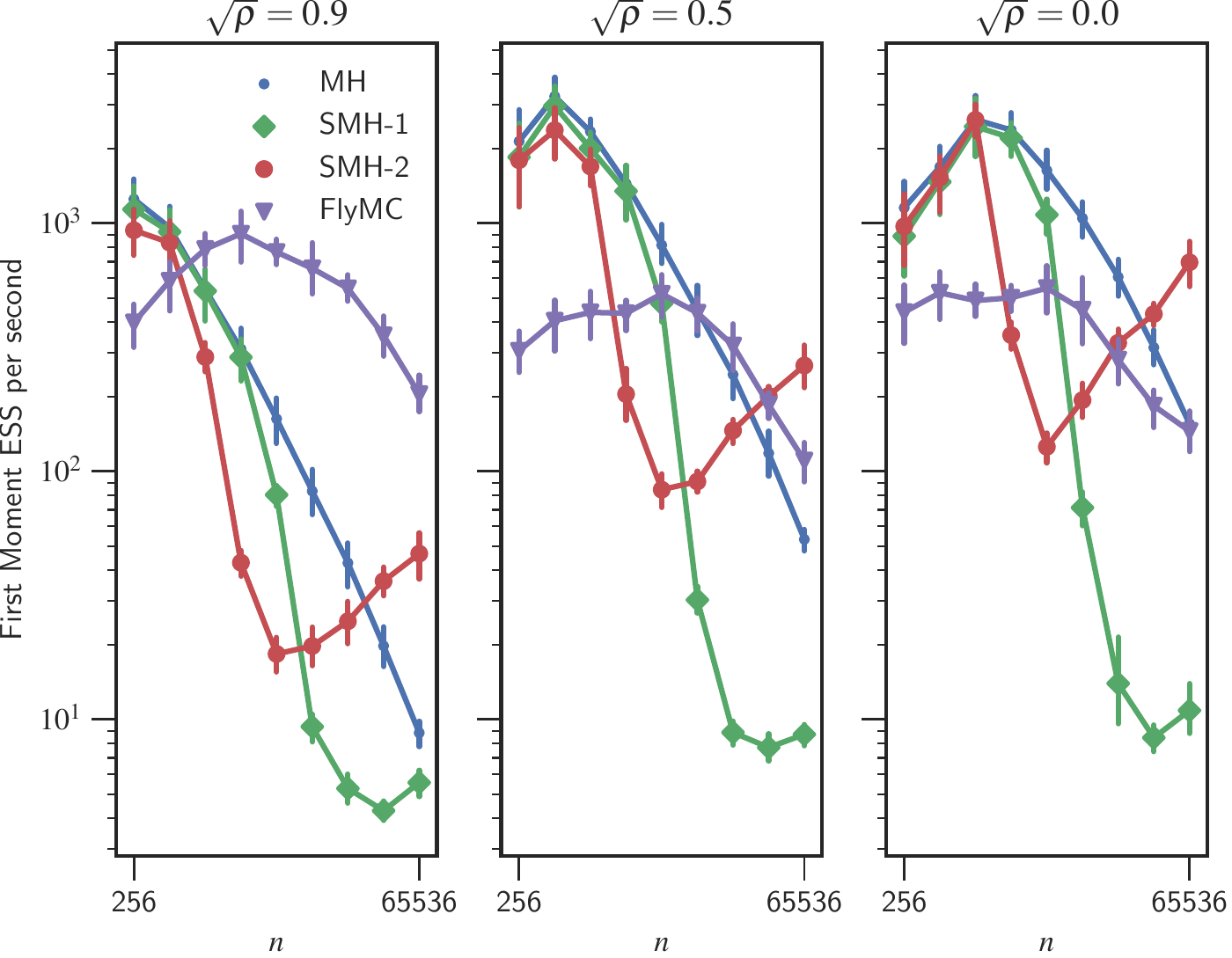}}
\caption{ESS of first regression coefficient for a logistic regression model of dimension 20.}
\label{fig:rlr_nonreversible_iact_20d}
\end{center}
\vskip -0.2in
\end{figure}

\subsection{Robust Linear Regression} \label{sec:robust-regression-bounds}

Here $x_i \in \R^d$ and $\data_i \in \R$. We use a flat prior $p(\state) \equiv 1$, and the likelihood is given by
\begin{eqnarray*}
    p(\data_i|\state, x_i) &=& \Student(y_i - \state^\top x_i \mid \nu).
\end{eqnarray*}
Here $\Student(\nu)$ denotes the Student-$t$ distribution with $\nu$ degrees of freedom that the user will specify. This gives
\[
    \pot_i(\state) = \frac{\nu+1}{2} \log \left(1 + \frac{(\data_i - \state^\top x_i)^2}{\nu} \right).
\]

To derive bounds necessary for \eqref{eq:grad-upper-bound}, let $\phi_i(\state) := \data_i - \state^T x_i$ and note that $\partial_j \phi_i(\state) = -x_{ij}$. Then we have
\begin{eqnarray*}
    \pot_i(\state) &=& \frac{\nu+1}{2} \log \left(1 + \frac{\phi_i(\state)^2}{\nu}\right) \\
    \partial_j \pot_i(\state) &=& -(\nu+1) x_{ij} \frac{\phi_i(\state)}{\nu + \phi_i(\state)^2} \\
    \partial_k \partial_j \pot_i(\state) &=& -(\nu+1) x_{ij} \frac{-x_{ik} (\nu + \phi_i(\state)^2) + 2 x_{ik} \phi_i(\state)^2}{\nu + \phi_i(\state)} \\
    &=& (\nu + 1) x_{ij} x_{ik} \frac{\nu - \phi_i(\state)^2}{(\nu + \phi_i(\state)^2)^2} \\
    \partial_\ell \partial_k \partial_j \pot_i(\state) &=& (\nu + 1) x_{ij} x_{ik} 
        \frac{2x_{i\ell} \phi_i(\state) (\nu + \phi_i(\state)^2)^2 + 4 x_{i\ell} (\nu - \phi_i(\state)^2)(\nu + \phi_i(\state)^2) \phi_i(\state)}{(\nu + \phi_i(\state)^2)^4} \\
    &=& -2 (\nu + 1) x_{ij} x_{ik} x_{i\ell} \frac{\phi_i(\state)( \phi_i(\state)^2 -3 \nu)}{(\nu + \phi_i(\state)^2)^3}
\end{eqnarray*}

In general,
\begin{eqnarray*}
    \sup_{t \in \R} \Abs{\frac{t}{\nu + t^2}} &=& \frac{1}{2\sqrt{\nu}} \\
    \sup_{t \in \R} \Abs{\frac{\nu - t^2}{(\nu + t^2)^2}} &=& \frac{1}{\nu} \\
    \sup_{t \in \R} \Abs{\frac{t(t^2 - 3\nu)}{(\nu + t^2)^3}} &=& \frac{3 + 2\sqrt{2}}{8\nu^{3/2}},
\end{eqnarray*}
so setting
\begin{eqnarray*}
    \potgradub_{1,i} &:=& \frac{\nu+1}{2\sqrt{\nu}} \max_{1\leq j \leq d} \abs{x_{ij}}\\
    \potgradub_{2,i} &:=& \frac{\nu+1}{\nu} \max_{1\leq j \leq d} \abs{x_{ij}}^2 \\
    \potgradub_{3,i} &:=& \frac{(\nu+1)(3 + 2\sqrt{2})}{4\nu^{3/2}} \max_{1\leq j \leq d} \abs{x_{ij}}^3
\end{eqnarray*}
satisfies \eqref{eq:grad-upper-bound}.

\begin{figure}[ht]
\vskip 0.2in
\begin{center}
\centerline{\includegraphics[width=0.5\columnwidth]{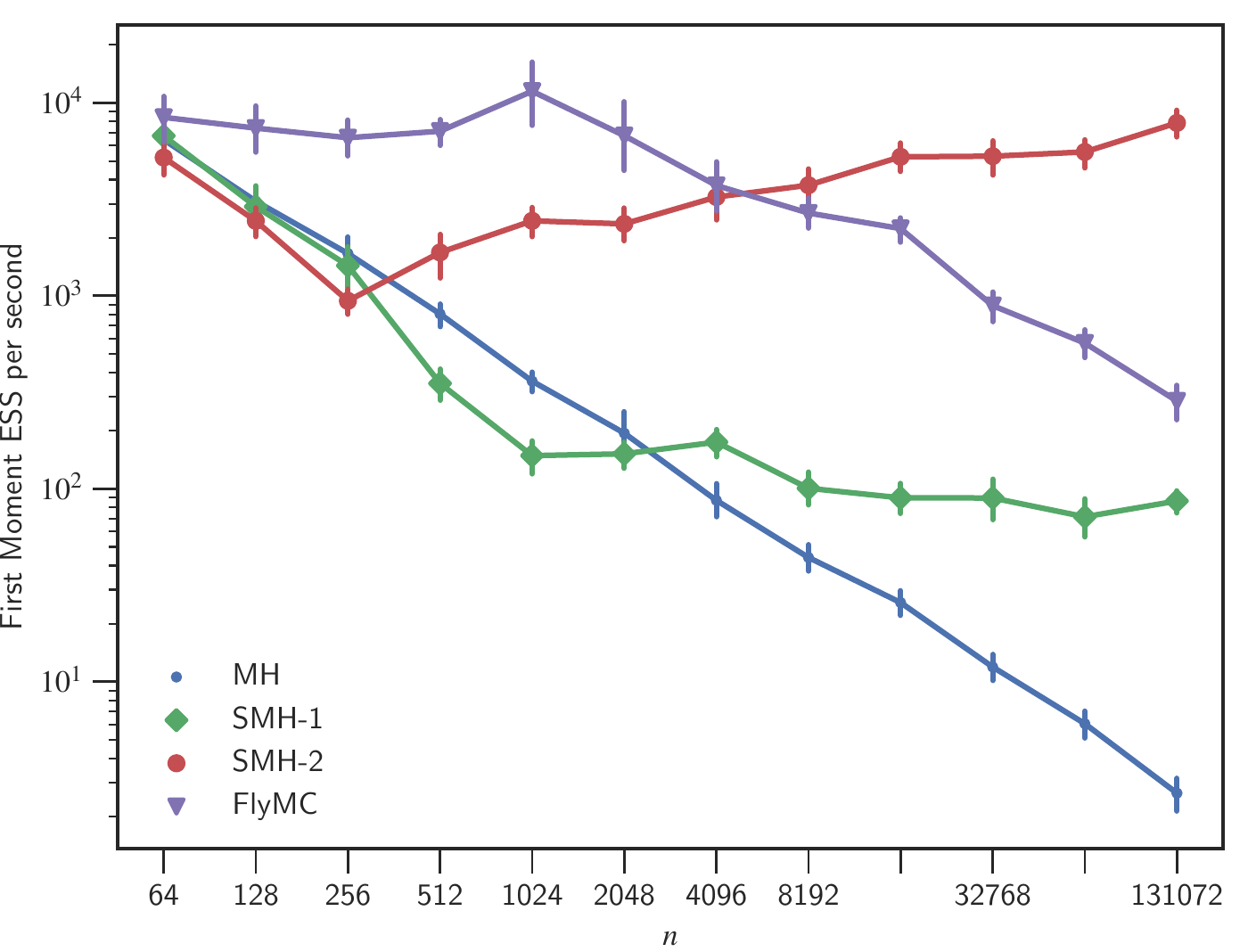}}
\caption{ESS for first regression coefficient of a robust linear regression posterior, scaled by execution time (higher is better).} 
\label{fig:rlr_nonreversible_iact}
\end{center}
\vskip -0.2in
\end{figure}

\begin{figure}[ht!]
\vskip 0.2in
\begin{center}
\centerline{\includegraphics[width=.5\columnwidth]{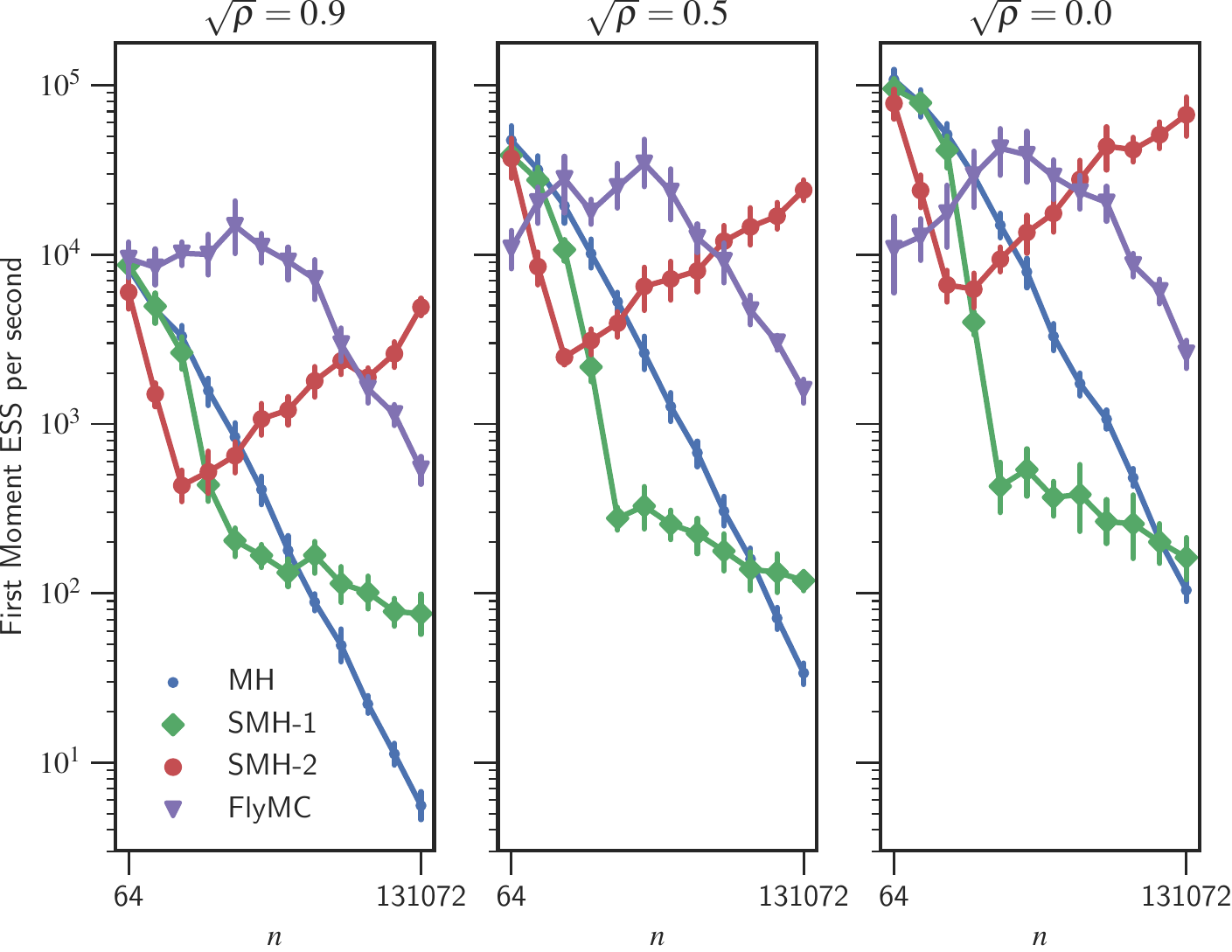}}
\caption{Effect of $\rho$ on ESS for first regression coefficient of the robust linear regression model, scaled by execution time (higher is better).}
\label{fig:rlr_reversible_iact}
\end{center}
\vskip -0.2in
\end{figure}

\begin{figure}[ht!]
\vskip 0.2in
\begin{center}
\centerline{\includegraphics[width=.5\columnwidth]{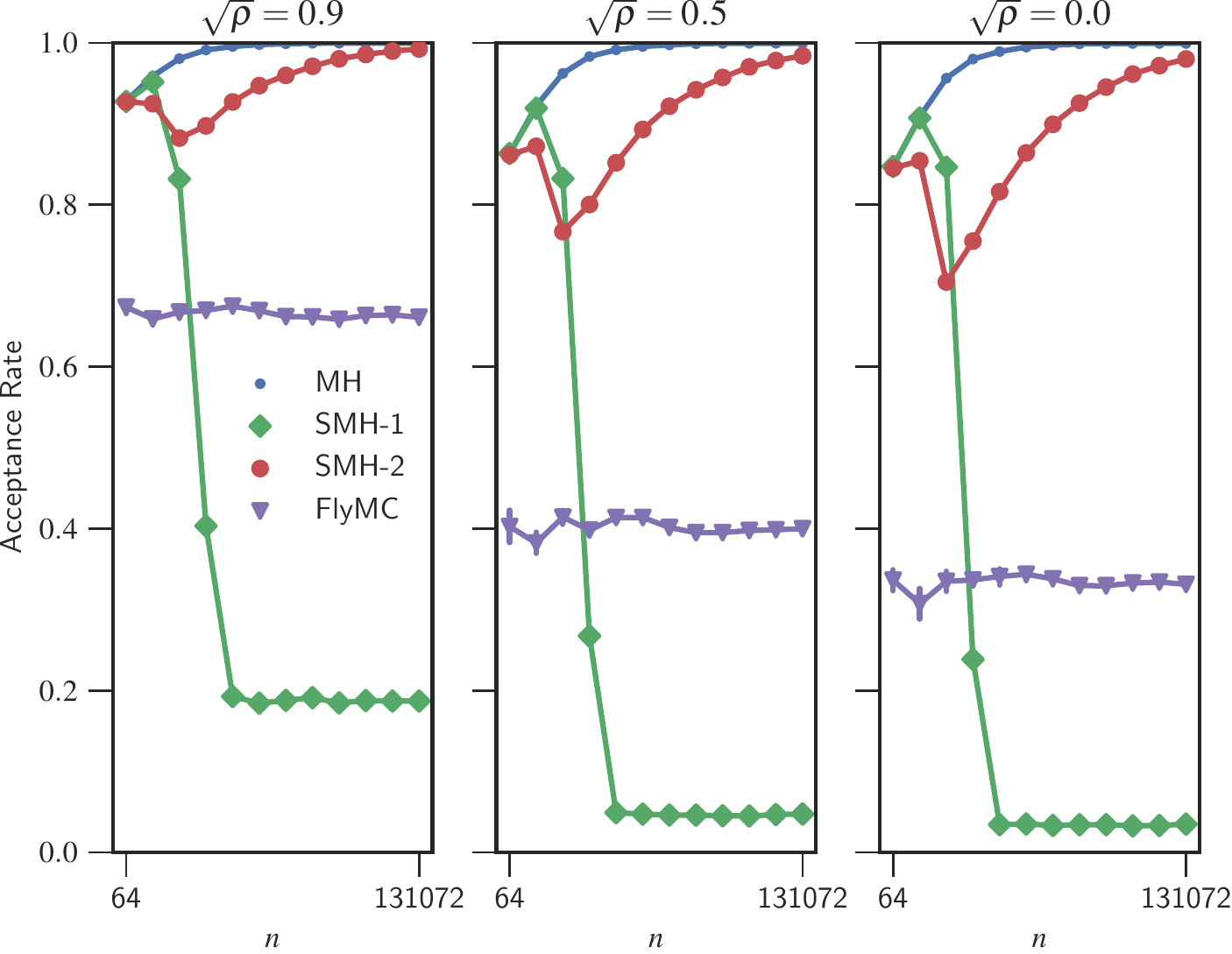}}
\caption{Acceptance rates for pCN proposals for the robust linear regression model.}
\label{fig:rlr_acceptance_rate}
\end{center}
\vskip -0.2in
\end{figure}

In Figure~\ref{fig:rlr_nonreversible_iact} we show effective sample size (ESS) per second for the robust linear regression model; this experiment mimics the conditions of Figure~\ref{fig:nonreversible_iact} in the main text, where we used a logistic regression model. The performance for this model is qualitatively similar to that for logistic regression. Figures~\ref{fig:rlr_reversible_iact} and \ref{fig:rlr_acceptance_rate} show the ESS and acceptance rate for pCN proposals as $\rho$ is varied. These mimic Figures~\ref{fig:reversible_iact} and \ref{fig:acceptance_rate} in the main text.
For these experiments we use synthetic data, taking an $n \times 10$ matrix $X$ with elements drawn independently from a standard normal distribution, and simulate $y_i = \sum_j X_{ij} + \epsilon$ where $\epsilon$ itself is drawn from a standard normal distribution. We choose as the model parameter $\nu=4.0$.

\end{document}